\documentclass{article}

    \PassOptionsToPackage{numbers, sort&compress}{natbib}
\usepackage[preprint]{neurips_2026}

\usepackage[utf8]{inputenc} 
\usepackage[T1]{fontenc}    
\usepackage{hyperref}       
\usepackage{url}            
\usepackage{booktabs}       
\usepackage{amsfonts}       
\usepackage{nicefrac}       
\usepackage{microtype}      
\usepackage{xcolor}         

\usepackage{wrapfig}
\usepackage{xcolor}         
\usepackage{caption}
\usepackage{tabularx}
\usepackage{enumitem}
\usepackage{float}

\usepackage{multirow}
\usepackage{kotex}
\usepackage{amsmath}
\usepackage{amssymb}
\usepackage{mathtools}
\usepackage{amsthm}
\usepackage{graphicx}

\usepackage{booktabs}
\usepackage{multirow}
\usepackage{graphicx}
\usepackage{bm}

\newcommand{\best}[1]{\boldmath\textbf{#1}\unboldmath}

\usepackage{algorithm}
\usepackage{algpseudocode}
\usepackage{float}

\theoremstyle{plain}
\newtheorem{theorem}{Theorem}[section]
\newtheorem{proposition}[theorem]{Proposition}
\newtheorem{lemma}[theorem]{Lemma}
\newtheorem{corollary}[theorem]{Corollary}
\theoremstyle{definition}

\theoremstyle{remark}
\newtheorem{remark}[theorem]{Remark}

\title{Rethinking Forward Processes for Score-Based Nonlinear Data Assimilation in High Dimensions}

\author{%
\setlength{\tabcolsep}{2.5em}
\begin{tabular}{cc}
\textbf{Eunbi Yoon} & \textbf{Won Chang} \\
{\normalfont KAIST} & {\normalfont Seoul National University} \\
\texttt{eunbiyoon6286@kaist.ac.kr} & \texttt{wonchang@snu.ac.kr}
\\[1.0em]
\textbf{Donghan Kim\thanks{Corresponding authors.}} 
& \textbf{Dae Wook Kim\textsuperscript{$\dagger$}} \\
{\normalfont KAIST} & {\normalfont KAIST} \\
\texttt{kimdonghan@kaist.ac.kr} & \texttt{daewook@kaist.ac.kr}
\end{tabular}
}
\begin{document}

\maketitle

\begin{abstract}
Data assimilation is the process of estimating the state of a dynamical system over time by combining model predictions with measurements.
This task becomes challenging when the system is nonlinear and high-dimensional. 
To address this, score-based Bayesian filters have recently emerged. 
However, these methods still show unsatisfactory performance in certain cases, particularly under spatially sparse measurements. 
Such degradation stems from heuristic approximations of the likelihood score, whose errors can accumulate over time.
This limitation arises because the methods simply adopt a classical forward process for generative modeling that transforms a data distribution toward a Gaussian distribution, which is independent of the measurement equation.
Here, we propose a forward process tailored for filtering that transforms the system state toward the measurement space, enabling a theoretically sound formulation of the likelihood score.
Based on this, we develop the Measurement-Aware Score-based Filter (MASF).
We evaluate MASF on Kolmogorov flow, a high-dimensional fluid benchmark with up to $\mathcal{O}(10^5)$ dimensions, under diverse measurement operators, including nonlinear cases with a dimensional mismatch between the state and the measurements.
MASF shows improved performance over existing score-based filters and ensemble-type Kalman filters.
Notably, MASF achieves up to a $28.2\times$ wall-clock speedup compared with the baselines when using amortized pretraining.
Our implementation is available at \texttt{https://github.com/tcnllab-oss/masf}.
\end{abstract}

\section{Introduction}
\label{sec:introduction}

Data assimilation estimates the state of a dynamical system over time by combining model predictions with measurements~\cite{evensen2009enkf,reich2015probabilistic}.
It arises in a broad range of domains where time-evolving dynamics must be inferred from incomplete measurements, including geophysical forecasting and biological processes~\cite{Chipilski2020PECANProfilers, Aksoy2009EnKF_Radar_PartI, Bao2021AlopeciaAssimilation}.
The data assimilation problem is typically solved using Bayesian filtering, which alternates between a time-update step that propagates the current state under the state equation and a measurement-update step that corrects the prediction using measurements from the measurement equation~\cite{sarkka2013bayesian,asch2016data}.
However, carrying out these updates exactly is rarely feasible in high-dimensional nonlinear systems because they involve integrals that typically lack closed-form expressions~\cite{sarkka2013bayesian,doucet2001smc}.

To circumvent this intractability, two major classes of approximate Bayesian filters have been proposed.
The first class consists of ensemble-based Gaussian filters, including the ensemble Kalman filter (EnKF) and the local ensemble transform Kalman filter (LETKF)~\cite{kalman1960new,Evensen2009EnKF_StateParam,Whitaker2002EnKFNoPertObs,hunt2007letkf}.
These methods approximate the posterior under Gaussian assumptions by propagating ensemble-based moment estimates.
Although computationally efficient, their accuracy can degrade when the posterior is strongly non-Gaussian or when the state and measurement equations are highly nonlinear~\cite{asch2016data,Houtekamer1998EnK}. 
The second class consists of particle filters. 
They approximate the posterior with weighted samples and can represent non-Gaussian distributions more flexibly, but suffer from weight degeneracy in high-dimensional settings~\cite{Gordon1993PFNovel,Andrieu2010PMCMC,arulampalam2002tutorial,snyder2008obstacles}.

To address these limitations, recent studies~\cite{rozet2023sda,bao2024sf,bao2024ensf,ding2025ssls} have incorporated score-based generative modeling into Bayesian filtering. 
This is motivated by the ability of score-based models to sample from complex high-dimensional distributions by learning score functions, i.e., gradients of log-densities~\cite{song2019generative,song2021score,dhariwal2021diffusion,hyvarinen2005score,vincent2011connection}. 
In score-based filtering, a score model learns the prior score on perturbed states generated by a forward process, and this prior score is combined with a likelihood score during the measurement-update step to generate posterior samples.

Although score-based filtering has shown meaningful success, there is still room for improvement. 
For example, the Score-based Filter (SF) uses a forward process to transform the state distribution into a Gaussian distribution~\cite{bao2024sf}.
This makes the likelihood score at perturbed states intractable. 
As a result, a heuristic single point-mass approximation of the likelihood score with a damping function is adopted, where the choice of the damping function largely determines the performance.
To mitigate this issue, Score-based Sequential Langevin Sampling (SSLS) has been proposed, which targets the posterior distribution via Langevin Monte Carlo~\cite{ding2025ssls}.
However, this Langevin-based sampling relies on annealing over
noise levels, which can substantially increase the number of sampling steps and make inference computationally demanding~\cite{song2021score}.

Furthermore, most importantly, existing score-based methods learn the prior score from a forward process that does not reflect the relationship between the system state and the measurements specified by the measurement equation. 
As a result, measurement information enters the measurement-update step only through a heuristic likelihood score.
This can lead to degraded performance, particularly under spatially sparse measurements, where the likelihood score may be zero on unmeasured coordinates or highly localized near measured regions.

Here, we propose the Measurement-Aware Score-based Filter (MASF), which incorporates the measurement equation directly into the forward process and formulates the likelihood score at perturbed states along the reverse-time sampling trajectory.
This likelihood-score formulation is exact for linear measurement operators and provides a principled approximation for nonlinear operators.
Furthermore, MASF can achieve lower computational cost than ensemble-based Kalman filters through amortized pretraining and a lightweight model.
This addresses the computational burden of retraining the score model at each assimilation step, which has long been a major obstacle to the scalability of score-based filtering. 
Our main contributions are as follows.

\begin{figure}[t]
    \centering
    \includegraphics[width=1\linewidth]{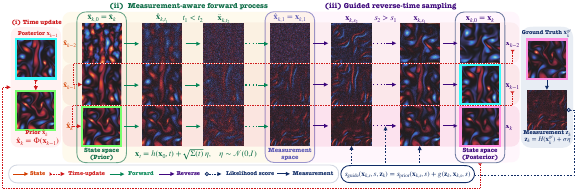}
    \caption{\textbf{Overview of MASF.}
    The measurement-aware forward process interpolates states toward the measurement space, while guided reverse-time sampling performs the measurement update.}
    \label{fig:masf_forward_process}
    \vspace{-0.5cm}
\end{figure}

\textbf{Main Contributions}
\begin{itemize}[leftmargin=*, itemsep=3pt, topsep=3pt, parsep=1pt]
    \item We introduce a measurement-aware forward process that interpolates between the state and measurement spaces according to the measurement equation. 
    This allows the learned prior score to reflect the state-measurement relationship; see Fig.~\ref{fig:masf_forward_process}.
    \item We derive an exact likelihood score for linear measurements and a theoretically grounded and practical approximation for nonlinear measurements.
    \item This derivation is achieved by projecting the non-Markovian stochastic differential equation (SDE) corresponding to the forward process onto a Markovian SDE with the same marginal distributions. 
    Notably, for linear measurement operators, we show that this Markovian projection can be realized as a linear SDE.
    \item We demonstrate on Kolmogorov flow that MASF improves accuracy and wall-clock efficiency and scales to $\mathcal{O}(10^5)$ dimensions under nonlinear and dimensionally mismatched measurements.
\end{itemize}

\section{Background}
In this section, we review the key concepts underlying score-based filters.
We first review Bayesian filtering for continuous-time dynamics with discrete-time measurements, emphasizing its recursive decomposition into time-update and measurement-update steps.
We then review score-based generative models formulated through classical SDEs, which provide a principled framework for learning and sampling from complex distributions.

\subsection{Bayesian Filtering}
\label{subsec:Bayesian Filtering}
Let $\tau$ denote the physical time variable associated with the underlying dynamical system.
Let $f:\mathbb{R}^d\times\mathbb{R}\to\mathbb{R}^d$ and
$g:\mathbb{R}^d\times\mathbb{R}\to\mathbb{R}^{d\times d}$ denote the drift and diffusion terms, respectively.
We consider a continuous-time state process $X_{\tau}\in\mathbb{R}^d$ governed by the SDE
\begin{align}
\label{eq:state_sde}
    dX_{\tau} = f(X_{\tau},\tau)\,d\tau + g(X_{\tau},\tau)\,dB_{\tau},
\end{align}
where $B_{\tau}$ denotes a $d$-dimensional Brownian motion.
Let $\{\tau_k\}_{k=0}^K$ denote discrete measurement times, and define $X_k := X_{\tau_k}$.
At each measurement time $\tau_k$, we observe a measurement $Z_k\in\mathbb{R}^m$ through the measurement equation
\begin{align}
\label{eq:measurement}
    Z_k = H(X_k) + \sigma\,\eta_k,
    \qquad
    \eta_k \sim \mathcal{N}(0,I_m),
\end{align}
where $H:\mathbb{R}^d \to \mathbb{R}^m$ denotes the measurement operator, $\sigma>0$ is the noise scale, and the identity matrix $I_m \in \mathbb{R}^{m \times m}$.
For notational convenience, we later embed the measurements into the state space, e.g., by zero-padding or an appropriate embedding map, and write
$Z_k\in\mathbb{R}^d$ and $H:\mathbb{R}^d\to\mathbb{R}^d$.

We now define the filtering objective.
Given measurements $\mathbf{z}_{1:k}=(\mathbf{z}_1,\dots,\mathbf{z}_k)$ up to time $\tau_k$, Bayesian filtering aims to estimate the posterior of the state $X_k$~\cite{sarkka2013bayesian}:
\begin{align*}
    p(\mathbf{x}_{k}\mid \mathbf{z}_{1:k})
    :=
    p\!\left(
        X_{k}=\mathbf{x}_{k}
        \,\middle|\,
        Z_1=\mathbf{z}_1,\dots,Z_k=\mathbf{z}_k
    \right).
\end{align*}
This posterior distribution is computed recursively by alternating time-update and measurement-update steps.

\paragraph{Time-update step.}
Given the posterior distribution $p(\mathbf{x}_{k-1}\mid \mathbf{z}_{1:k-1})$ at time $\tau_{k-1}$, the state equation \eqref{eq:state_sde} induces the transition density from $X_{k-1}$ to $X_k$:
\begin{align}
\label{eq:transition_density}
    p(\mathbf{x}_{k}\mid \mathbf{x}_{k-1})
    :=
    p\!\left(
        X_k=\mathbf{x}_{k}
        \,\middle|\,
        X_{k-1}=\mathbf{x}_{k-1}
    \right).
\end{align}
Using this transition density, the prior distribution $p(\mathbf{x}_{k}\mid \mathbf{z}_{1:k-1})$ at time $\tau_k$ is obtained by the Chapman--Kolmogorov equation~\cite{law2015da}:
\begin{align}
\label{eq:time_update}
    p(\mathbf{x}_{k}\mid \mathbf{z}_{1:k-1})
    =
    \int
    p(\mathbf{x}_{k}\mid \mathbf{x}_{k-1})
    \,
    p(\mathbf{x}_{k-1}\mid \mathbf{z}_{1:k-1})
    \,
    d\mathbf{x}_{k-1}.
\end{align}
To obtain samples from this prior distribution, the state SDE \eqref{eq:state_sde} is solved over the interval $[\tau_{k-1},\tau_k]$ using a numerical SDE solver, such as the Euler--Maruyama method~\cite{maruyama1955continuous}.

\paragraph{Measurement-update step.}
Given the prior distribution $p(\mathbf{x}_{k}\mid \mathbf{z}_{1:k-1})$, the new measurement $\mathbf{z}_k$ is incorporated through the likelihood
$p(\mathbf{z}_k\mid \mathbf{x}_k)$.
Under the measurement equation \eqref{eq:measurement}, this likelihood is given by
\[
p(\mathbf z_k\mid \mathbf x_k)
=
\mathcal N(\mathbf z_k;H(\mathbf x_k),\sigma^2 I_d).
\]
Combining the prior distribution with this likelihood, Bayes' rule~\cite{sarkka2013bayesian} gives the posterior distribution:
\begin{align}
\label{eq:posterior}
\underbrace{p(\mathbf{x}_{k} \mid \mathbf{z}_{1:k})}_{\text{Posterior}}
\;\propto\;
\underbrace{p(\mathbf{x}_{k} \mid \mathbf{z}_{1:k-1})}_{\text{Prior}}
\;\underbrace{p(\mathbf{z}_k \mid \mathbf{x}_{k})}_{\text{Likelihood}}.
\end{align}
Taking the gradient of the log-posterior $\log p(\mathbf{x}_{k} \mid \mathbf{z}_{1:k})$ with respect to $\mathbf{x}_k$ yields the additive score decomposition:
\begin{align}
\label{eq:posterior_score}
    \nabla_{\mathbf{x}_{k}}\log p(\mathbf{x}_{k} \mid \mathbf{z}_{1:k})
    =
    \nabla_{\mathbf{x}_{k}} \log p(\mathbf{x}_{k} \mid \mathbf{z}_{1:k-1})
    +
    \nabla_{\mathbf{x}_{k}} \log p(\mathbf{z}_k \mid \mathbf{x}_{k}).
\end{align}
We refer to the gradients of the log-posterior, log-prior, and log-likelihood as the posterior score, prior score, and likelihood score, respectively.

\subsection{Score-based generative models}
We review score-based generative models formulated through SDEs.
Let $t\in[0,1]$ denote the diffusion time.
We consider a linear SDE commonly used in score-based generative modeling~\cite{song2021score}:
\begin{align}
\label{eq:linear_sde}
    dX_t = F(t)X_t\,dt + G(t)\,dB_t,\qquad t\in[0,1],
\end{align}
where $F(t)\in\mathbb{R}^{d\times d}$ and $G(t)\in\mathbb{R}^{d\times d}$ denote the time-dependent drift and diffusion terms, respectively.
A widely used instance of \eqref{eq:linear_sde} is the variance-preserving (VP) SDE, defined by
\begin{align}
\label{eq:vp_sde}
    F(t) = -\tfrac{1}{2}\beta(t)I_d,
    \qquad
    G(t) = \sqrt{\beta(t)}\,I_d,
\end{align}
where $\beta: [0,1] \to \mathbb{R}_{\ge 0}$ is an increasing noise schedule~\cite{nichol2021improved}.
For this VP SDE, the solution $X_t$ of \eqref{eq:vp_sde} has the closed-form representation
\begin{align}
\label{eq:forward_marginal}
    X_t = a(t)X_0 + \gamma(t)\,\eta,
\end{align}
where $\eta \sim\mathcal{N}(0,I_d)$ and
\begin{align}
\label{eq:schedule_coeffs}
    \frac{d}{dt}\log a(t) = -\tfrac{1}{2}\beta(t),
    \qquad
    \gamma^2(t)=1-a^2(t).
\end{align}
The representation implies that the conditional score
$\nabla_{\mathbf{x}_t}\log p(\mathbf{x}_t\mid \mathbf{x}_0)$ is linear in $\mathbf{x}_t$.
The resulting linear form allows score models to be trained efficiently even in high-dimensional settings via the denoising score-matching objective~\cite{hyvarinen2005score,vincent2011connection}.

\section{Methods}
\label{sec:method}
MASF consists of three main components.
First, MASF propagates posterior samples from the previous physical time to obtain prior samples at the current measurement time (Fig.~\ref{fig:masf_forward_process}(i)).
Second, these prior samples are transformed toward the measurement space (Fig.~\ref{fig:masf_forward_process}(ii); Section~\ref{subsec:measurement_aware_forward_process}).
Third, the transformed samples are transported back to the state space through a reverse-time SDE guided by the estimated likelihood score (Fig.~\ref{fig:masf_forward_process}(iii); Section~\ref{subsec:likelihood_score_estimation}).
For reference, Table~\ref{tab:notation} summarizes the notation used throughout the paper, and Appendix~\ref{app:theory} provides proofs for all results in Section~\ref{sec:method}.

\subsection{Measurement-Aware Forward Process}
\label{subsec:measurement_aware_forward_process}

We construct a measurement-aware forward process $X_t$, where $t\in[0,1]$ denotes the diffusion time of this forward process, distinct from the physical time $\tau$ in the state SDE \eqref{eq:state_sde}.
Specifically, the forward process is controlled by a decreasing interpolation schedule $a:[0,1]\to[0,1]$ that satisfies $a(0)=1$ and $a(1)=0$.
Based on this schedule, the covariance schedule is defined as $\Sigma(t)=\sigma^2(1-a^2(t))I_d$, so that $\Sigma(0)=0$ and $\Sigma(1)=\sigma^2 I_d$.

Now, we define the interpolation map $h:\mathbb{R}^d\times[0,1]\to\mathbb{R}^d$ by
\begin{align}
\label{eq:interpolation_map}
    h(\mathbf{x},t)
    =
    a(t)\mathbf{x}
    +
    \bigl(1-a(t)\bigr)H(\mathbf{x}).
\end{align}
Using this map, we construct the measurement-aware forward process as
\begin{align}
\label{eq:forward_process}
    X_t
    =
    h(X_0,t)
    +
    \sqrt{\Sigma(t)}\,\eta,
    \qquad
    \eta\sim\mathcal{N}(0,I_d).
\end{align}
By construction, the process starts from the state and terminates in distribution at the measurement:
\begin{align}
    X_0 = X,
    \qquad
    X_1 = H(X)+\sigma\eta \stackrel{d}{=} Z,
\end{align}
where $\stackrel{d}{=}$ denotes equality in distribution.
Given the forward process in \eqref{eq:forward_process}, we next obtain an SDE whose marginal distributions match those of this process.

\begin{theorem}[Non-Markovian SDE representation]
\label{thm:non_markovian_sde}
Assume that $a(t)$ and $\Sigma(t)$ are differentiable, and that $\dot{\Sigma}(t)$ is positive semidefinite.
Let $\widetilde X_0\stackrel{d}{=}X_0$ be independent of the Brownian motion $B_t$.
Suppose that $\widetilde X_t$ satisfies
\begin{align}
\label{eq:non_markovian_sde}
    d\widetilde X_t
    =
    \partial_t h(\widetilde X_0,t)\,dt
    +
    \sqrt{\dot{\Sigma}(t)}\,dB_t,
    \qquad t\in[0,1].
\end{align}
Then $\mathcal{L}(\widetilde X_t)=\mathcal{L}(X_t)$ for all $t\in[0,1]$, where $\mathcal{L}(Y)$ denotes the law of a random variable $Y$.
\end{theorem}

Theorem~\ref{thm:non_markovian_sde} shows that the marginal distributions of the forward process in \eqref{eq:forward_process} can be realized by an SDE.
However, the SDE in \eqref{eq:non_markovian_sde} is non-Markovian because its drift depends on the initial variable $\widetilde X_0$.
This dependence prevents us from applying Anderson's reverse-time theorem~\cite{Anderson1982ReversetimeDE}.
To address this, we apply the Markovian projection theorem~\cite{brunick2013mimicking} to the SDE in \eqref{eq:non_markovian_sde}.

\begin{theorem}[Markovian projection]
\label{thm:markovian_projection}
Let $\widetilde X_t$ be defined by \eqref{eq:non_markovian_sde}, and assume that
\begin{align}
\label{eq:markovian_projection_integrability}
    \mathbb{E}
    \int_0^1
    \left(
        \left\|\partial_t h(\widetilde X_0,t)\right\|
        +
        \left\|\dot{\Sigma}(t)\right\|
    \right)
    dt
    <\infty.
\end{align}
Then there exists a measurable function $b:\mathbb{R}^d\times[0,1]\to\mathbb{R}^d$ such that
\begin{align}
\label{eq:markovian_projection_drift}
    b(\mathbf{x},t)
    :=
    \mathbb{E}\!\left[
        \partial_t h(\widetilde X_0,t)
        \mid
        \widetilde X_t=\mathbf{x}
    \right].
\end{align}
Moreover, there exists a Markovian process $\bar X_t$ satisfying
\begin{align}
\label{eq:markovian_projection_sde}
    d\bar X_t
    =
    b(\bar X_t,t)\,dt
    +
    \sqrt{\dot{\Sigma}(t)}\,dB_t,
    \qquad
    \bar X_0\stackrel{d}{=}X_0,
\end{align}
with $\bar X_0$ independent of the Brownian motion $B_t$.
Then $
    \mathcal{L}(\bar X_t)
    =
    \mathcal{L}(\widetilde X_t)
    =
    \mathcal{L}(X_t),
    \;
    t\in[0,1]$.
\end{theorem}

\subsection{Estimation of the Likelihood Score}
\label{subsec:likelihood_score_estimation}

Theorem~\ref{thm:markovian_projection} provides a Markovian SDE with the same marginal distributions as the measurement-aware forward process~\eqref{eq:forward_process}.
Its endpoint transition law will serve as the basis for estimating the likelihood score, but this requires an explicit form of the projected drift in \eqref{eq:markovian_projection_drift}.
We obtain such a representation through the following two conditional means:
\begin{align}
\label{eq:m1_m2_definition}
    m_1(\mathbf{x},t)
    &:=
    \mathbb{E}\!\left[\widetilde X_0 \mid \widetilde X_t=\mathbf{x}\right],
    \qquad
    m_2(\mathbf{x},t)
    :=
    \mathbb{E}\!\left[H(\widetilde X_0) \mid \widetilde X_t=\mathbf{x}\right].
\end{align}
Here, $m_1(\mathbf{x},t)$ estimates the initial state associated with the perturbed state $\mathbf{x}$, whereas $m_2(\mathbf{x},t)$ estimates the corresponding clean measurement.
These two conditional means determine the drift as
\begin{align}
\label{eq:projected_drift_m1_m2}
    b(\mathbf{x},t)
    =
    \dot a(t)
    \left(
        m_1(\mathbf{x},t)-m_2(\mathbf{x},t)
    \right).
\end{align}
Because this drift is generally nonlinear, the endpoint likelihood $p(\mathbf{z}\mid \mathbf{x}_t)$ cannot be evaluated in closed form.
We therefore approximate the endpoint transition law $p(\mathbf{z}\mid \mathbf{x}_t)$ of the projected SDE, rather than approximating the measurement operator $H$.
Specifically, we integrate the projected SDE~\eqref{eq:markovian_projection_sde} from $t$ to $1$:
\begin{align}
\label{eq:endpoint_transition_integral}
    \bar X_1
    =
    \bar X_t
    +
    \int_t^1
    \dot a(u)
    \left(
        m_1(\bar X_u,u)-m_2(\bar X_u,u)
    \right)\,du
    +
    \int_t^1
    \sqrt{\dot\Sigma(u)}\,dB_u .
\end{align}
We then approximate the drift integral by freezing the conditional means at $(\bar X_t,t)$, which gives
\begin{align}
\label{eq:endpoint_transition_approx}
    \bar X_1
    \approx
    \bar X_t
    -
    a(t)
    \left(
        m_1(\bar X_t,t)-m_2(\bar X_t,t)
    \right)
    +
    \sqrt{\Sigma_{t\to 1}}\,\eta,
    \qquad
    \eta\sim\mathcal{N}(0,I),
\end{align}
where $\Sigma_{s\to t}:=\Sigma(t)-\Sigma(s)$. 
Since the endpoint marginal satisfies $\bar X_1\stackrel{d}{=}Z$, this approximation defines a surrogate likelihood for $\mathbf z$ given $\mathbf x_t$.

\begin{proposition}[Endpoint Gaussian approximation for the likelihood score]
\label{prop:endpoint_gaussian_likelihood_score}
For $0\le t<1$, define
\begin{align}
\label{eq:mu_t_definition}
    \mu_t(\mathbf{x})
    :=
    \mathbf{x}
    -
    a(t)
    \left(
        m_1(\mathbf{x},t)-m_2(\mathbf{x},t)
    \right).
\end{align}
Assume that $\mu_t$ is differentiable with respect to $\mathbf{x}$.
Under the endpoint approximation in \eqref{eq:endpoint_transition_approx}, the conditional likelihood is approximated by
\begin{align}
\label{eq:gaussian_guidance_approx}
    p(\mathbf{z}\mid \mathbf{x}_t=\mathbf{x})
    \approx
    \mathcal{N}
    \left(
        \mathbf{z};
        \mu_t(\mathbf{x}),
        \Sigma_{t\to 1}
    \right).
\end{align}
Consequently, the corresponding likelihood-score approximation is
\begin{align}
\label{eq:nonlinear_guidance_score}
    g(\mathbf{z},\mathbf{x},t)
    &:=
    \nabla_{\mathbf{x}}
    \log
    \mathcal{N}
    \left(
        \mathbf{z};
        \mu_t(\mathbf{x}),
        \Sigma_{t\to 1}
    \right) \nonumber \\
    &=
    \left(\nabla_{\mathbf{x}}\mu_t(\mathbf{x})\right)^{\mathsf T}
    \Sigma_{t\to 1}^{-1}
    \left(
        \mathbf{z}-\mu_t(\mathbf{x})
    \right).
\end{align}
\end{proposition}

Proposition~\ref{prop:endpoint_gaussian_likelihood_score} applies to general nonlinear measurement operators.
For linear measurement operators, the projected SDE~\eqref{eq:markovian_projection_sde} admits a linear Markov realization with the same Fokker--Planck equation.

\begin{theorem}[Closed-form likelihood score for linear measurements]
\label{thm:linear_measurement_case}
Suppose that $H(\mathbf{x})=A\mathbf{x}$ for some matrix $A\in\mathbb{R}^{d\times d}$, and define
\begin{align}
\label{eq:At_definition}
    A(t)
    =
    a(t)I+\bigl(1-a(t)\bigr)A.
\end{align}
Then
\begin{align}
    h(\mathbf{x},t)=A(t)\mathbf{x},
    \qquad
    m_2(\mathbf{x},t)=A\,m_1(\mathbf{x},t).
\end{align}
Assume that $A(t)$ is invertible for $t\in[0,1)$ and that the corresponding covariance is positive semidefinite.
Then the projected SDE for $\bar X_t$ admits a linear Markov realization $\check X_t$ with transition kernel
\begin{align}
\label{eq:linear_transition_kernel}
    \check X_t\mid \check X_s
    \sim
    \mathcal{N}
    \left(
        M_{s\to t}\check X_s,\,
        \check \Sigma_{s\to t}
    \right),
    \qquad 0\le s<t\le 1,
\end{align}
where
\begin{align}
\label{eq:linear_transition_params}
    M_{s\to t}
    =
    A(t)A(s)^{-1},
    \qquad
    \check \Sigma_{s\to t}
    =
    \Sigma(t)
    -
    M_{s\to t}\Sigma(s) M_{s\to t}^{\mathsf T}.
\end{align}
Consequently, for $0\le t<1$, the likelihood score is
\begin{align}
\label{eq:linear_likelihood_score}
    \nabla_{\mathbf{x}_t}\log p(\mathbf{z}\mid \mathbf{x}_t)
    =
    M_{t\to 1}^{\mathsf T}
    \check \Sigma_{t\to 1}^{-1}
    \left(
        \mathbf{z}
        -
        M_{t\to 1}\mathbf{x}_t
    \right).
\end{align}
\end{theorem}

\subsection{Guided Reverse-Time Sampling}
We construct a sampler from the measurement space to the state space using the reverse-time SDE of~\eqref{eq:markovian_projection_sde}. 
The drift of the reverse-time SDE contains an additional score term; choosing this term as the posterior score yields samples from the posterior distribution. 
Based on the posterior-score decomposition~\eqref{eq:posterior_score}, the reverse-time drift can be expressed in terms of the prior and the likelihood score~\eqref{eq:nonlinear_guidance_score}.
Since the likelihood score is singular at $t=1$, we start from $1-\varepsilon$ for some $\varepsilon >0$ and discretize the posterior-score guided reverse SDE.

\begin{lemma}[Guided reverse-time sampler]
\label{lem:guided_sampler}
Let $s>t$ be two times along the reverse-time trajectory.
Given $\mathbf{x}_s$, let
\begin{align}
\label{eq:hhat_s_definition}
    \widehat h_s(\mathbf{x}_s)
    :=
    a(s)m_1(\mathbf{x}_s,s)
    +
    \bigl(1-a(s)\bigr)m_2(\mathbf{x}_s,s).
\end{align}
Using Tweedie's formula~\cite{efron2011tweedie}, the prior score at $s$ is estimated as
\begin{align}
\label{eq:prior_score_estimator_m}
    s_{\mathrm{prior}}(\mathbf{x}_s,s)
    &:=
    -\Sigma(s)^{-1}
    \bigl(
        \mathbf{x}_s-\widehat h_s(\mathbf{x}_s)
    \bigr).
\end{align}
Based on the posterior-score decomposition in \eqref{eq:posterior_score}, we approximate the posterior score by combining the prior and likelihood scores:
\begin{align}
\label{eq:guided_score}
    s_{\mathrm{guide}}(\mathbf{x}_s,s,\mathbf{z})
    &:=
    s_{\mathrm{prior}}(\mathbf{x}_s,s)
    +
    g(\mathbf{z},\mathbf{x}_s,s).
\end{align}
Then, for $\eta\sim\mathcal{N}(0,I)$, the reverse-time update is
\begin{align}
\label{eq:guided_sampler_step}
    \mathbf{x}_t
    =
    \mathbf{x}_s
    +
    \bigl(a(t)-a(s)\bigr)
    \bigl(
        m_1(\mathbf{x}_s,s)-m_2(\mathbf{x}_s,s)
    \bigr)
    -
    \Sigma_{s\to t}
    s_{\mathrm{guide}}(\mathbf{x}_s,s,\mathbf{z})
    +
    \sqrt{-\Sigma_{s\to t}}\,\eta.
\end{align}
\end{lemma}
To improve numerical stability, we balance the prior and likelihood scores using time-dependent weights, following guidance-based diffusion sampling~\cite{ho2022cfg,chung2023dps}:
\begin{align}
\label{eq:scaled_guided_score}
    s_{\mathrm{guide}}(\mathbf{x}_s,s,\mathbf{z})
    =
    \lambda_s(s)\,s_{\mathrm{prior}}(\mathbf{x}_s,s)
    +
    \lambda_g(s)\,g(\mathbf{z},\mathbf{x}_s,s).
\end{align}
When $\lambda_s=\lambda_g=1$, \eqref{eq:scaled_guided_score} recovers the formal posterior-score decomposition.

To implement the sampler, we need to estimate the conditional means $m_1$ and $m_2$, which we learn from samples of the forward process in \eqref{eq:forward_process} via denoising objectives with $L_2$ losses.
\begin{proposition}[Denoising objective for conditional means]
\label{prop:training_conditional_means}
Let $m_{1,\theta}$ and $m_{2,\theta}$ be parametric estimators of
$m_1$ and $m_2$ in \eqref{eq:m1_m2_definition}.
For $i=1,2$, let $y_1(\mathbf{x}_0)=\mathbf{x}_0$ and
$y_2(\mathbf{x}_0)=H(\mathbf{x}_0)$, and define
\begin{align}
\label{eq:m1_m2_losses_compact}
    L_i(\theta,t)
    &:=
    \mathbb{E}_{\mathbf{x}_t\sim \tilde p_t}
    \left[
        \left\|
            m_{i,\theta}(\mathbf{x}_t,t)-m_i(\mathbf{x}_t,t)
        \right\|^2
    \right],
    \\
    \widehat L_i(\theta,t)
    &:=
    \mathbb{E}_{\mathbf{x}_0\sim p_0}
    \mathbb{E}_{\mathbf{x}_t\sim p_t(\cdot\mid \mathbf{x}_0)}
    \left[
        \left\|
            m_{i,\theta}(\mathbf{x}_t,t)-y_i(\mathbf{x}_0)
        \right\|^2
    \right],
\end{align}
where $\tilde p_t$ is the marginal distribution of $\widetilde X_t$ in \eqref{eq:non_markovian_sde} and
$p_t(\cdot\mid\mathbf{x}_0)$ is the conditional distribution of $X_t$ in
\eqref{eq:forward_process}. 
Then
\begin{align}
    \nabla_\theta L_i(\theta,t)
    =
    \nabla_\theta \widehat L_i(\theta,t),
    \qquad i=1,2.
\end{align}
Thus, in practice, we train $m_{1,\theta}$ and $m_{2,\theta}$ by minimizing
$\mathcal{L}_{\mathrm{reg}}(\theta,t)
    =
    \widehat L_1(\theta,t)+\widehat L_2(\theta,t).$
\end{proposition}

Details of the overall procedure are provided in Appendix~\ref{app:algorithmic_details}.
Additional implementation details, including the noise schedule and guidance weights, are described in Appendix~\ref{app:noise_schedule_guidance_weights}.

\section{Experimental Setup}
\label{subsec:experimental_setup}
We evaluated MASF on Kolmogorov flow under various linear and nonlinear measurement operators.
This setup provides a challenging testbed for high-dimensional nonlinear data assimilation, with Kolmogorov flow at state dimensions exceeding $10^5$ and diverse measurement settings, including spatially sparse measurements.
To specify this setup, we describe the dynamical system, the model architecture, and the implementation details.

\textbf{Kolmogorov flow.}
We use a two-dimensional Kolmogorov flow as a high-dimensional benchmark.
In Kolmogorov flow, the state $\mathbf{u}(\tau)\in\mathbb{R}^{2\times H\times W}$ is a velocity field with horizontal and vertical components on a periodic grid.
The dynamics follow the incompressible Navier--Stokes equations with external forcing $\mathbf{f}$:
\begin{align}
    \partial_t \mathbf{u}
    &=
    -(\mathbf{u}\cdot\nabla)\mathbf{u}
    +
    \frac{1}{\mathrm{Re}}\nabla^2\mathbf{u}
    -
    \frac{1}{\rho}\nabla p
    +
    \mathbf{f},
    \qquad
    \nabla\cdot\mathbf{u}=0.
\end{align}
We simulate the system using JAX-CFD~\citep{kochkov2021mlcfd} on $[0,2\pi]^2$ with $\rho=1$, $\mathrm{Re}=1000$, and time step $dt=0.2$.
Unless otherwise stated, measurements are taken every $10$ steps over steps $50$--$100$, with measurement noise $\sigma=0.1$.

\textbf{Measurement operators.}
We consider two linear and two nonlinear measurement operators.
Linear measurement operators are grid masks and center masks, denoted by $A_{\mathrm{grid}}$ and $A_{\mathrm{center}}$, respectively.
Both are binary masks applied element-wise to the state:
\begin{align}
[A_{\mathrm{grid}}]_{ij}
&=
\mathbf{1}\{i\equiv0 \pmod{s},\ j\equiv0 \pmod{s}\},
\qquad
[A_{\mathrm{center}}]_{ij}
=
\mathbf{1}\{(i,j)\notin\mathcal{C}_h\},
\end{align}
where $s$ is the grid-mask stride and $\mathcal{C}_h$ is the central square hole of side length $h$.
Nonlinear measurements include an element-wise sigmoid measurement and a speed measurement.
The element-wise sigmoid preserves the state dimension but applies the sigmoid function to each component.
In contrast, the speed measurement maps the two velocity channels at each spatial location to bounded scalar measurements by applying the sigmoid function to the squared local speed:
\begin{align}
[H(\mathbf{x})]_{ij}
=
\operatorname{sigmoid}\!\left(\sum_{c=1}^{2}x_{cij}^{2}\right).
\end{align}
The squared local speed is a standard quantity derived from velocity fields in fluid dynamics~\citep{kundu2015fluid,pope2000turbulent}.
The effective dimensions of all measurements are summarized in Table~\ref{tab:measurement_settings}.
\begin{table}[H]
\centering
\vspace{-0.5cm}
\caption{
\textbf{Measurement settings.}
Dim. denotes the effective dimension of all measurements.
}
\label{tab:measurement_settings}
\small
\setlength{\tabcolsep}{6pt}
\renewcommand{\arraystretch}{1.12}
\resizebox{\linewidth}{!}{%
\begin{tabular}{llll}
\toprule
Measurement & Dim. & Property & $H(\mathbf{x})$ \\
\midrule
Grid mask
& $2(H/s)(W/s)$
& Spatially sparse mask
& $A_{\mathrm{grid}}\odot\mathbf{x}$ \\

Center mask
& $2(HW-h^2)$
& Central hole mask
& $A_{\mathrm{center}}\odot\mathbf{x}$ \\

Sigmoid
& $2HW$
& Element-wise nonlinear
& $\operatorname{sigmoid}(\mathbf{x})$ \\

Speed
& $HW$
& Channel-coupled nonlinear
& $[H(\mathbf{x})]_{ij}
= \operatorname{sigmoid}\!\left(\sum_{c=1}^{2}x_{cij}^{2}\right)$ \\
\bottomrule
\end{tabular}%
}
\vspace{-0.5cm}
\end{table}

\textbf{Model architecture.}
For linear measurements, we use a U-Net~\cite{ronneberger2015unet} to estimate $m_{1,\theta}$.
Since $H(\mathbf{x})=A\mathbf{x}$ in this case, the second conditional mean is available analytically as $m_{2,\theta}=A m_{1,\theta}$ and does not require separate models.
For nonlinear measurements, we use a dual-head U-Net with a shared encoder and separate decoders for $m_{1,\theta}$ and $m_{2,\theta}$~\cite{ronneberger2015unet,kuga2017multi}.
The same backbone as MASF is used for the score-based baselines whenever applicable to ensure a fair comparison.

\textbf{Efficient implementation.}
MASF uses a lightweight model, amortized pretraining, and a reduced number of function evaluations (NFE).
The model is pretrained once on $1000$ samples generated from the dynamics and fine-tuned at each measurement update.
In all experiments, MASF uses ensemble size $N=10$ and $\mathrm{NFE}=150$--$250$.

\textbf{Baselines and metrics.}
We compare MASF with EnKF, LETKF, SF, and SSLS.
Each baseline is tuned over its main hyperparameters, including ensemble size, inflation, localization, and sampling parameters.
We report root mean squared error (RMSE), critical success index (CSI)~\cite{schaefer1990critical}, and wall-clock time averaged over the assimilation window.
For MASF, wall-clock time includes online fine-tuning and guided sampling at measurement updates, but excludes one-time offline pretraining, since pretraining is performed without any measurement updates and the resulting initialization can be reused across assimilation runs with the same dynamics and measurement setting.
Additional implementation details, including hyperparameters, tuning procedures, and pretraining effects, are provided in Appendices~\ref{app:masf_tuning}--\ref{app:pretraining_effect}.

\begin{table}[t]
\centering
\vspace{-0.25cm}
\caption{
\textbf{Main results on Kolmogorov flow at $128^2$ resolution.}
Mean $\pm$ standard deviation over five seeds.
Wall-clock time is reported in seconds, and speedup is computed relative to MASF within each measurement setting.
}
\label{tab:main_results_selected}

\begingroup
\scriptsize
\renewcommand{\arraystretch}{1}
\setlength{\aboverulesep}{0.08ex}
\setlength{\belowrulesep}{0.08ex}
\setlength{\cmidrulesep}{0.02ex}
\setlength{\tabcolsep}{5pt}

\resizebox{0.98\textwidth}{!}{%
\begin{tabular}{cllccccc}
\toprule

Type & Measurement & Method & Ensemble & RMSE $(\downarrow)$ & CSI $(\uparrow)$ & Wall-clock (s) & Speedup \\
\midrule
\multirow{10}{*}{Linear}
& \multirow{5}{*}{Grid mask}
& EnKF  & 250 & $0.18 \pm 0.04$ & $0.76 \pm 0.05$ & $180.3 \pm 0.4$ & $1.5\times$ \\
& & LETKF & 40  & $0.18 \pm 0.01$ & $0.74 \pm 0.02$ & $389.6 \pm 5.4$ & $3.3\times$ \\
& & SF    & 100 & $0.81 \pm 0.04$ & $0.00 \pm 0.00$ & $670.6 \pm 0.9$ & \best{$5.6\times$} \\
& & SSLS  & 100 & $0.47 \pm 0.07$ & $0.34 \pm 0.10$ & $611.5 \pm 1.1$ & $5.1\times$ \\
& & MASF  & 10  & \best{$0.15 \pm 0.02$} & \best{$0.78 \pm 0.03$} & \best{$118.8 \pm 2.5$} & -- \\
\cmidrule(lr){2-8}
& \multirow{5}{*}{Center mask}
& EnKF  & 250 & $0.75 \pm 0.17$ & $0.32 \pm 0.08$ & $185.4 \pm 2.5$ & $1.1\times$ \\
& & LETKF & 40  & $0.34 \pm 0.05$ & $0.64 \pm 0.03$ & $2294.1 \pm 45.2$ & \best{$14.0\times$} \\
& & SF    & 100 & $0.28 \pm 0.02$ & $0.68 \pm 0.03$ & $672.3 \pm 3.4$ & $4.1\times$ \\
& & SSLS  & 100 & $0.29 \pm 0.05$ & $0.69 \pm 0.09$ & $610.8 \pm 0.8$ & $3.7\times$ \\
& & MASF  & 10  & \best{$0.27 \pm 0.06$} & \best{$0.71 \pm 0.04$} & \best{$164.1 \pm 1.9$} & -- \\
\midrule
\multirow{10}{*}{Nonlinear}
& \multirow{5}{*}{Sigmoid}
& EnKF  & 250 & $0.62 \pm 0.13$ & $0.36 \pm 0.03$ & $191.1 \pm 1.1$ & $2.1\times$ \\
& & LETKF & 40  & $0.18 \pm 0.04$ & $0.77 \pm 0.03$ & $2586.4 \pm 27.4$ & \best{$28.2\times$} \\
& & SF    & 100 & $0.26 \pm 0.07$ & $0.63 \pm 0.05$ & $559.1 \pm 3.8$ & $6.1\times$ \\
& & SSLS  & 100 & $0.11 \pm 0.03$ & $0.82 \pm 0.05$ & $536.3 \pm 1.2$ & $5.8\times$ \\
& & MASF  & 10  & \best{$0.07 \pm 0.01$} & \best{$0.89 \pm 0.01$} & \best{$91.7 \pm 1.5$} & -- \\
\cmidrule(lr){2-8}
& \multirow{5}{*}{Speed}
& EnKF  & 250 & $1.30 \pm 0.05$ & $0.04 \pm 0.02$ & $190.7 \pm 0.7$ & $1.5\times$ \\
& & LETKF & 40  & $0.96 \pm 0.08$ & $0.09 \pm 0.02$ & $2567.6 \pm 9.6$ & \best{$20.7\times$} \\
& & SF    & 100 & $0.87 \pm 0.04$ & $0.01 \pm 0.01$ & $380.5 \pm 7.1$ & $3.1\times$ \\
& & SSLS  & 100 & $0.83 \pm 0.06$ & $0.01 \pm 0.01$ & $358.9 \pm 2.2$ & $2.9\times$ \\
& & MASF  & 10  & \best{$0.23 \pm 0.04$} & \best{$0.67 \pm 0.08$} & \best{$123.9 \pm 1.0$} & -- \\
\bottomrule
\end{tabular}%
}
\endgroup
\vspace{-0.5cm}
\end{table}

\section{Experimental Results}
\label{subsec:main_filtering_results}
We evaluate MASF on Kolmogorov flow in terms of filtering accuracy, high-dimensional scalability, and runtime--accuracy trade-offs.
We first compare MASF with EnKF, LETKF, SF, and SSLS across the four measurement operators in Table~\ref{tab:measurement_settings}.
We then examine scaling to $256^2$ and $512^2$ resolutions and analyze sensitivity to runtime budget, temporal sparsity, and spatial sparsity.

\subsection{Main Results at $128^2$ Resolution}
\label{subsec:main_results}
Table~\ref{tab:main_results_selected} reports the main results on Kolmogorov flow, and Fig.~\ref{fig:various} shows qualitative comparisons; see Appendix~\ref{app:qualitative_samples} for additional qualitative samples.
Across all measurement operators, MASF achieves the best RMSE and CSI while maintaining favorable wall-clock time.
We provide a measurement-wise analysis of the results below.

\textbf{Grid Mask.}
Under grid-mask measurements, the performance of SF and SSLS degrades because their likelihood guidance,
$A^{\mathsf T}(\mathbf{z}-A\mathbf{x})/\sigma^2$, only injects residual information at observed grid locations.
For the sparse mask $A$, this term is zero on unobserved coordinates, so the sampler receives no direct measurement correction in most of the state space.
MASF avoids this issue by deriving the likelihood score for the perturbed state $\mathbf{x}_t$, achieving the best RMSE and CSI with the shortest wall-clock time.

\textbf{Center Mask.}
Center-mask measurements are less sparse than grid-mask measurements, so SF and SSLS receive denser likelihood-score information and become more competitive.
MASF still achieves the best RMSE and fastest wall-clock time with comparable CSI.
LETKF becomes substantially slower because the average number of measurements within each localization window increases from $6.5$ to $475.5$, a trend also observed in the sigmoid and speed settings.

\textbf{Sigmoid and Speed.}
For the element-wise sigmoid measurement, several baselines remain competitive, but MASF achieves the best RMSE and CSI with the shortest wall-clock time, yielding a $28.2\times$ speedup over LETKF.
The speed measurement is more challenging because it nonlinearly couples velocity channels. 
In this setting, EnKF and LETKF show lower accuracy than MASF, while SF and SSLS yield near-zero CSI despite moderate RMSE, suggesting poor recovery of high-speed regions.
As shown in Fig.~\ref{fig:various}, MASF recovers posterior states more accurately through nonlinear measurements.

\begin{figure}[t]
    \centering
    \includegraphics[width=0.95\linewidth]{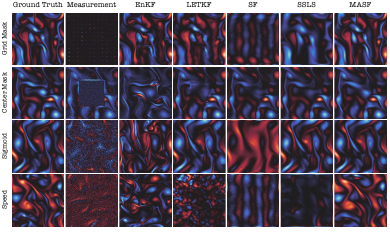}
    \caption{
    \textbf{Qualitative results.}
    Rows correspond to different measurements, and columns show the ground truth, measurement, and estimates from EnKF, LETKF, SF, SSLS, and MASF.
    }
    \label{fig:various}
    \vspace{-0.5cm}
\end{figure}

\subsection{Scaling to High Resolution}
We assess high-dimensional scaling under grid-mask measurements at $256^2$ and $512^2$ resolutions, corresponding to state dimensions above $10^5$.
At $256^2$, MASF achieves the best RMSE and CSI while being $3.0\times$--$4.8\times$ faster than the baselines.
At $512^2$, MASF remains stable and achieves the best accuracy, whereas SF and SSLS are omitted due to numerical instability; MASF is $10.6\times$ faster than EnKF and $3.9\times$ faster than LETKF.

\begin{table}[htb!]
\centering
\setlength{\tabcolsep}{5pt}
\renewcommand{\arraystretch}{0.95}
\resizebox{0.95\textwidth}{!}{%
\scriptsize
\begin{tabular}{clccccc}
\toprule
Resolution & Method & Ensemble & RMSE $(\downarrow)$ & CSI $(\uparrow)$ & Wall-clock (s) & Speedup \\
\midrule
\multirow{5}{*}{$256^2$}
& EnKF  & 250 & $0.17 \pm 0.03$ & $0.76 \pm 0.03$ & $1276.3 \pm 5.8$ & $3.7\times$ \\
& LETKF & 40  & $0.19 \pm 0.01$ & $0.70 \pm 0.01$ & $1646.7 \pm 15.2$ & \best{$4.8\times$} \\
& SF    & 100 & $0.88 \pm 0.03$ & $0.00 \pm 0.00$ & $1056.2 \pm 4.3$ & $3.1\times$ \\
& SSLS  & 100 & $0.72 \pm 0.06$ & $0.03 \pm 0.04$ & $1016.7 \pm 8.8$ & $3.0\times$ \\
& MASF  & 10  & \best{$0.16 \pm 0.02$} & \best{$0.77 \pm 0.03$} & \best{$343.2 \pm 4.0$} & - \\
\midrule
\multirow{3}{*}{$512^2$}
& EnKF  & 250 & $0.18 \pm 0.07$ & $0.78 \pm 0.05$ & $11162.5 \pm 80.3$ & \best{$10.6\times$} \\
& LETKF & 40  & $0.30 \pm 0.02$ & $0.46 \pm 0.05$ & $4140.1 \pm 37.4$ & $3.9\times$ \\
& MASF  & 10  & \best{$0.13 \pm 0.01$} & \best{$0.81 \pm 0.02$} & \best{$1049.4 \pm 22.1$} & - \\
\bottomrule
\end{tabular}%
}
\caption{
\textbf{High-dimensional scaling on grid-mask measurements.}
Mean $\pm$ standard deviation over five seeds.
Wall-clock time is reported in seconds, and speedup is computed relative to MASF at the same resolution.
}
\label{tab:high_dimensional_scaling}
\vspace{-0.5cm}
\end{table}

\subsection{Runtime--Accuracy and Sensitivity Analysis}
\label{subsec:runtime_sensitivity}

Fig.~\ref{fig:sensitivity} summarizes runtime--accuracy and sensitivity results under grid-mask measurements.
SF is omitted for visual clarity because its RMSE is substantially larger in this setting.
MASF consistently achieves low RMSE with substantially shorter wall-clock time than SSLS.
EnKF improves with larger ensembles but requires more runtime and shows higher variability, while LETKF is competitive but slower than MASF.
For temporal-length sensitivity, we vary the final assimilation step with the temporal gap fixed at $15$.
As the assimilation horizon increases, MASF maintains the lowest RMSE, while EnKF error grows and SSLS improves gradually.
For temporal and spatial sparsity, MASF degrades more slowly than the baselines and remains the most accurate.
In contrast, SSLS is comparable to EnKF at grid-mask stride $5$ but degrades sharply from stride $10$.
Additional ensemble-size and sensitivity results are provided in Appendices~\ref{app:ensemble_size_sensitivity} and~\ref{app:extended_sensitivity}.

\begin{figure}[htb!]
    \centering
    \includegraphics[width=1\linewidth]{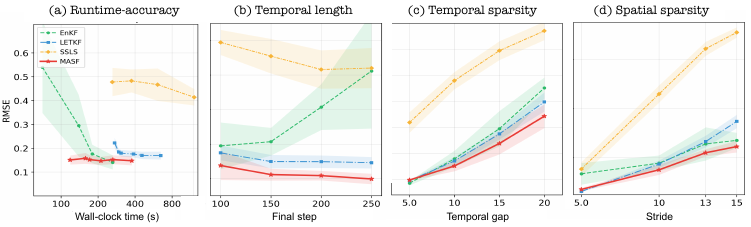}
    \caption{
    \textbf{Runtime--accuracy and sensitivity analysis.}
    Comparison under grid-mask measurements across runtime budget, temporal length, temporal gap, and spatial sparsity.
    Shaded regions denote standard deviation over five seeds.}
    \label{fig:sensitivity}
    \vspace{-0.5cm}
\end{figure}

\section{Limitations and Broader Impacts}
\textbf{Limitations.}
MASF assumes a known state SDE and measurement equation, and may degrade under misspecification or extremely noisy measurements.
For nonlinear measurements, it relies on an endpoint Gaussian approximation, whose error analysis is left for future work.
MASF also requires online fine-tuning and depends on hyperparameters such as guidance scales, normalization, NFE, and fine-tuning epochs.
Our experiments focus on Kolmogorov flow with five seeds and scaling up to $512^2$ resolution; broader systems, larger 3D domains, and more extensive evaluation remain future work.

\textbf{Broader Impacts.}
MASF may benefit scientific and engineering applications requiring accurate, efficient state estimation from partial or nonlinear measurements, such as fluid monitoring, forecasting, and sensor-based control.
Its reduced online cost may make high-dimensional filtering more practical.
In high-stakes settings, misspecified dynamics, biased measurements, or overconfident estimates may lead to unreliable decisions, requiring validation, uncertainty assessment, and appropriate privacy and security safeguards.

\section{Conclusion}
\label{sec:conclusion}
We propose the Measurement-Aware Score-based Filter (MASF), which redesigns the forward process itself for score-based data assimilation by incorporating the measurement equation.
Unlike classical forward processes that perturb data toward a noise distribution, the measurement-aware forward process lets the prior score be learned on perturbed states that reflect the state-measurement relationship.
From this forward process, we derive a Markovian SDE with matching marginal distributions, whose endpoint transition law yields likelihood scores on perturbed states: exact for linear measurements and theoretically grounded through an endpoint approximation for nonlinear measurements.
With amortized pretraining and a lightweight model, MASF improves accuracy and online wall-clock efficiency on Kolmogorov flow, particularly under spatially sparse and channel-coupled nonlinear measurements where existing filters perform poorly, while remaining effective at state dimensions exceeding $10^5$.
These results position MASF as a framework for score-based data assimilation that combines efficiency and scalability in nonlinear high-dimensional systems.

\section{Funding}
\label{sec:funding}
This work was supported by the National Research Foundation of Korea (Grant No. RS-2023-00301976, RS-2025-00523567, RS-2025-00561696, RS-2025-02215354, RS-2026-25512726, and RS-2026-25522728), the Korea Basic Science Institute (Grant No. RS-2026-25500300), the New Faculty Startup Fund from KAIST (Grant No. G04240060), the KAIST-CERAGEM Research Fund (Grant No. G01250193), and the New Faculty Startup Fund from Seoul National University (Grant No. 326-20240027).

\bibliographystyle{unsrtnat}
\bibliography{main}

\clearpage
\appendix

\paragraph{Notation.}
Table~\ref{tab:notation} summarizes the main symbols used throughout the paper.

\begin{table}[h]
\vspace{-0.5cm}
\centering
\caption{\textbf{Summary of notation.}}
\label{tab:notation}
\renewcommand{\arraystretch}{1.25}
\begin{tabular}{p{0.16\linewidth} p{0.58\linewidth} p{0.18\linewidth}}
\hline
\textbf{Symbols} & \textbf{Description} & \textbf{Defined in} \\
\hline

$\tau$ 
& The physical time variable of the dynamical system. 
& Eq.~\eqref{eq:state_sde} \\

$t\in[0,1]$ 
& The diffusion time variable of the measurement-aware forward process. 
& Sec.~\ref{subsec:measurement_aware_forward_process} \\

\hline

$X_k $ 
& The state at the measurement time $\tau_k$. 
& Sec.~\ref{subsec:Bayesian Filtering} \\

$Z_k$ 
& The measurement at the measurement time $\tau_k$. 
& Eq.~\eqref{eq:measurement} \\

$\mathbf{z}_{1:k}$ 
& The observed measurement sequence $(\mathbf{z}_1,\ldots,\mathbf{z}_k)$. 
& Sec.~\ref{subsec:Bayesian Filtering} \\

\hline

$X_t$ 
& The measurement-aware forward process. 
& Eq.~\eqref{eq:forward_process} \\

$\widetilde X_t$ 
& The non-Markovian SDE of the forward process. 
& Eq.~\eqref{eq:non_markovian_sde} \\

$\bar X_t$ 
& The Markovian projection. 
& Eq.~\eqref{eq:markovian_projection_sde} \\

$\check X_t$ 
& The linear Markov realization. 
& Eq.~\eqref{eq:linear_transition_kernel} \\

\hline

$H$ 
& The measurement operator. 
& Eq.~\eqref{eq:measurement} \\

$a(t)$ 
& The interpolation schedule with $a(0)=1$ and $a(1)=0$. 
& Sec.~\ref{subsec:measurement_aware_forward_process} \\

$h(\mathbf{x},t)$ 
& The interpolation map between the identity map and the measurement operator. 
& Eq.~\eqref{eq:interpolation_map} \\

$\Sigma(t)$ 
& The noise covariance schedule, $\Sigma(t)=\sigma^2(1-a^2(t))I$. 
& Sec.~\ref{subsec:measurement_aware_forward_process} \\

$\Sigma_{s\to t}$ 
& The covariance increment $\Sigma_{s\to t}:=\Sigma(t)-\Sigma(s)$. 
& Sec.~\ref{subsec:likelihood_score_estimation} \\

\hline

$b(\mathbf{x},t)$ 
& The projected Markov drift. 
& Eq.~\eqref{eq:markovian_projection_drift} \\

$m_1(\mathbf{x},t)$ 
& The conditional mean of the initial state. 
& Eq.~\eqref{eq:m1_m2_definition} \\

$m_2(\mathbf{x},t)$ 
& The conditional mean of the measurement. 
& Eq.~\eqref{eq:m1_m2_definition} \\

\hline

$\mu_t(\mathbf{x})$ 
& The mean of the endpoint Gaussian approximation. 
& Eq.~\eqref{eq:mu_t_definition} \\

$g(\mathbf{z},\mathbf{x},t)$ 
& The likelihood-score approximation for nonlinear measurements. 
& Eq.~\eqref{eq:nonlinear_guidance_score} \\

$s_{\mathrm{prior}}$ 
& The prior score estimator. 
& Eq.~\eqref{eq:prior_score_estimator_m} \\

$s_{\mathrm{guide}}$ 
& The guided posterior score estimator. 
& Eq.~\eqref{eq:guided_score} \\

\hline

$A$ 
& The matrix representation of a linear measurement operator.
& Thm.~\ref{thm:linear_measurement_case}\\

$A(t)$ 
& The linear interpolation matrix $A(t)=a(t)I+(1-a(t))A$. 
& Eq.~\eqref{eq:At_definition} \\

$M_{t\to 1}$ 
& The transition matrix from diffusion time $t$ to the endpoint in the linear case. 
& Eq.~\eqref{eq:linear_transition_params} \\

$\check \Sigma_{t\to 1}$ 
& The endpoint transition covariance for the linear Markov realization. 
& Eq.~\eqref{eq:linear_transition_params} \\

\hline
\end{tabular}
\vspace{-0.2cm}
\end{table}

\section{Theoretical Details}
\label{app:theory}

This section develops the SDE underlying MASF.
We first define a measurement-aware forward process $X_t$, then construct a non-Markovian SDE $\widetilde X_t$ with the same marginal distribution, and obtain a Markovian SDE $\bar X_t$ via Markovian projection.
The projected process provides the basis for endpoint likelihood-score estimation.
For nonlinear measurements, we approximate the endpoint transition law locally by a Gaussian distribution.
For linear measurements, the projected process admits a linear Markov realization, yielding a closed-form endpoint transition and likelihood score.
The resulting likelihood score is combined with the prior score to deduce the guided reverse-time sampler.

\subsection{Measurement-Aware Forward Process}
\label{app:measurement_aware_forward}

Let $a:[0,1]\to[0,1]$ be a decreasing schedule with $a(0)=1$ and $a(1)=0$, and define
\begin{align}
\label{eq:masf_sigma_t_definition}
   \Sigma(t)=\sigma^2(1-a^2(t))I.
\end{align}
For a measurement operator $H$, define
\begin{align}
\label{eq:masf_interpolation_map}
    h(\mathbf{x},t)
    =
    a(t)\mathbf{x}
    +
    \bigl(1-a(t)\bigr)H(\mathbf{x}).
\end{align}
Then, the measurement-aware forward process is
\begin{align}
\label{eq:masf_forward_process}
    X_t
    =
    h(X_0,t)
    +
    \sqrt{\Sigma(t)}\,\eta,
    \qquad
    \eta\sim\mathcal{N}(0,I),
    \qquad t\in[0,1].
\end{align}
By construction,
\begin{align}
    X_0=X,
    \qquad
    X_1=H(X_0)+\sigma\eta\stackrel{d}{=}Z.
\end{align}

\begin{theorem}[Non-Markovian SDE representation]
\label{thm:masf_non_markovian_sde}
Assume that $a(t)$ and $\Sigma(t)$ are differentiable and $\dot{\Sigma}(t)$ is positive semidefinite.
Let $\widetilde X_0\stackrel{d}{=}X_0$ be independent of the Brownian motion $B_t$.
Define
\begin{align}
\label{eq:masf_non_markovian_sde}
    d\widetilde X_t
    =
    \partial_t h(\widetilde X_0,t)\,dt
    +
    \sqrt{\dot{\Sigma}(t)}\,dB_t,
    \qquad t\in[0,1].
\end{align}
Then $\widetilde X_t$ has the same marginal distributions as the measurement-aware forward process:
\begin{align}
\label{eq:masf_non_markovian_same_marginal}
    \mathcal{L}(\widetilde X_t)=\mathcal{L}(X_t),
    \qquad t\in[0,1].
\end{align}
\end{theorem}

\begin{proof}
Conditioned on $\widetilde X_0$, integrating \eqref{eq:masf_non_markovian_sde} gives
\begin{align}
    \widetilde X_t
    =
    h(\widetilde X_0,t)
    +
    \int_0^t \sqrt{\dot{\Sigma}(u)}\,dB_u.
\end{align}
The stochastic integral is Gaussian with covariance
\begin{align}
    \int_0^t \dot{\Sigma}(u)\,du
    =
    \Sigma(t)-\Sigma(0)
    =
    \Sigma(t),
\end{align}
since $\Sigma(0)=0$.
Therefore,
\begin{align}
    \widetilde X_t
    =
    h(\widetilde X_0,t)
    +
    \sqrt{\Sigma(t)}\,\eta,
    \qquad
    \eta\sim\mathcal{N}(0,I).
\end{align}
Because $\widetilde X_0\stackrel{d}{=}X_0$, this representation has the same marginal law as
\begin{align}
    X_t=h(X_0,t)+\sqrt{\Sigma(t)}\,\eta.
\end{align}
Hence $\mathcal{L}(\widetilde X_t)=\mathcal{L}(X_t)$ for all $t\in[0,1]$.
\end{proof}

\subsection{Markovian Projection}
\label{app:markovian_projection}

The process $\widetilde X_t$ has the desired marginal distributions but is generally non-Markovian because its drift depends on the initial variable $\widetilde X_0$.
We therefore use Markovian projection to construct a Markov process with the same one-time marginal distributions.

\begin{theorem}[Markovian projection]
\label{thm:masf_markovian_projection}
Let $\widetilde X_t$ be the non-Markovian process in \eqref{eq:masf_non_markovian_sde}, and assume that the regularity and integrability conditions for Markovian projection hold.
Define
\begin{align}
\label{eq:masf_markovian_projection_drift}
    b(\mathbf{x},t)
    :=
    \mathbb{E}\!\left[
        \partial_t h(\widetilde X_0,t)
        \mid
        \widetilde X_t=\mathbf{x}
    \right].
\end{align}
Consider the Markovian SDE
\begin{align}
\label{eq:masf_markovian_projection_sde}
    d\bar X_t
    =
    b(\bar X_t,t)\,dt
    +
    \sqrt{\dot{\Sigma}(t)}\,dB_t,
    \qquad
    \bar X_0\stackrel{d}{=}X_0,
\end{align}
with $\bar X_0$ independent of $B_t$.
Then $\bar X_t$, $\widetilde X_t$, and $X_t$ have identical one-time marginal distributions for all $t\in[0,1]$:
\begin{align}
\label{eq:masf_markovian_projection_same_marginal}
    \mathcal{L}(\bar X_t)
    =
    \mathcal{L}(\widetilde X_t)
    =
    \mathcal{L}(X_t),
    \qquad t\in[0,1].
\end{align}
\end{theorem}

\begin{proof}
This is a direct application of the Markovian projection theorem of
\citet{brunick2013mimicking}.
The non-Markovian process \eqref{eq:masf_non_markovian_sde} has drift
$\partial_t h(\widetilde X_0,t)$ and diffusion coefficient
$\sqrt{\dot{\Sigma}(t)}$.
Markovian projection replaces the drift by its conditional expectation given the current state while preserving the diffusion coefficient.
Thus the projected drift is precisely \eqref{eq:masf_markovian_projection_drift}.
The projected Markovian process has the same one-time marginal distributions as $\widetilde X_t$.
Combining this with Theorem~\ref{thm:masf_non_markovian_sde} gives the stated marginal equality.
\end{proof}

\subsection{Projected Drift and Likelihood-Score Approximation}
\label{app:likelihood_score_approx}

We first express the projected drift in terms of conditional means.

\begin{proposition}[Projected drift representation]
\label{prop:masf_projected_drift}
Define
\begin{align}
\label{eq:masf_m1_m2_definition}
    m_1(\mathbf{x},t)
    :=
    \mathbb{E}\!\left[\widetilde X_0 \mid \widetilde X_t=\mathbf{x}\right],
    \qquad
    m_2(\mathbf{x},t)
    :=
    \mathbb{E}\!\left[H(\widetilde X_0) \mid \widetilde X_t=\mathbf{x}\right].
\end{align}
Then the projected drift in \eqref{eq:masf_markovian_projection_drift} is
\begin{align}
\label{eq:masf_projected_drift_m1_m2}
    b(\mathbf{x},t)
    =
    \dot a(t)
    \left(
        m_1(\mathbf{x},t)-m_2(\mathbf{x},t)
    \right).
\end{align}
\end{proposition}

\begin{proof}
From the definition of $h$,
\begin{align}
    \partial_t h(\mathbf{x},t)
    =
    \dot a(t)
    \left(
        \mathbf{x}-H(\mathbf{x})
    \right).
\end{align}
Substituting this into the projected drift gives
\begin{align}
    b(\mathbf{x},t)
    &=
    \mathbb{E}\!\left[
        \dot a(t)
        \left(
            \widetilde X_0-H(\widetilde X_0)
        \right)
        \mid
        \widetilde X_t=\mathbf{x}
    \right]
    \\
    &=
    \dot a(t)
    \left(
        \mathbb{E}[\widetilde X_0\mid \widetilde X_t=\mathbf{x}]
        -
        \mathbb{E}[H(\widetilde X_0)\mid \widetilde X_t=\mathbf{x}]
    \right),
\end{align}
which gives the claim.
\end{proof}

For nonlinear measurement operators, the exact transition law from $X_t$ to the endpoint $Z=X_1$ is generally intractable under the projected SDE.
We therefore use a local Gaussian approximation for the endpoint likelihood.

\begin{lemma}[Endpoint Gaussian approximation]
\label{lem:masf_endpoint_gaussian_approx}
Let $0\le t<1$ and suppose that the conditional means in the projected drift are frozen at $(\mathbf{x},t)$ over the interval $[t,1]$.
Then
\begin{align}
\label{eq:masf_endpoint_gaussian_approx}
    Z \mid X_t=\mathbf{x}
    \approx
    \mathcal{N}
    \left(
        \mu_t(\mathbf{x}),
        \Sigma_{t\to 1}
    \right),
\end{align}
where
\begin{align}
\label{eq:masf_mu_t_definition}
    \mu_t(\mathbf{x})
    &:=
    \mathbf{x}
    -
    a(t)
    \left(
        m_1(\mathbf{x},t)-m_2(\mathbf{x},t)
    \right),
    \\
    \Sigma_{t\to 1}
    &:=
    \Sigma(1)-\Sigma(t).
\end{align}
\end{lemma}

\begin{proof}
The projected SDE is
\begin{align}
    d\bar X_u
    =
    b(\bar X_u,u)\,du
    +
    \sqrt{\dot{\Sigma}(u)}\,dB_u.
\end{align}
Using Proposition~\ref{prop:masf_projected_drift} and freezing
$m_1,m_2$ at $(\mathbf{x},t)$ gives
\begin{align}
    b(\bar X_u,u)
    \approx
    \dot a(u)
    \left(
        m_1(\mathbf{x},t)-m_2(\mathbf{x},t)
    \right).
\end{align}
Thus,
\begin{align}
    \int_t^1 b(\bar X_u,u)\,du
    \approx
    \left(a(1)-a(t)\right)
    \left(
        m_1(\mathbf{x},t)-m_2(\mathbf{x},t)
    \right)
    =
    -a(t)
    \left(
        m_1(\mathbf{x},t)-m_2(\mathbf{x},t)
    \right).
\end{align}
The diffusion accumulated over $[t,1]$ has covariance
\begin{align}
    \int_t^1 \dot{\Sigma}(u)\,du
    =
    \Sigma(1)-\Sigma(t).
\end{align}
Combining the drift and diffusion contributions yields the Gaussian approximation in \eqref{eq:masf_endpoint_gaussian_approx}.
\end{proof}

\begin{corollary}[Likelihood-score approximation for nonlinear measurements]
\label{cor:masf_nonlinear_likelihood_score}
Under the endpoint Gaussian approximation in Lemma~\ref{lem:masf_endpoint_gaussian_approx}, the likelihood score is approximated by
\begin{align}
\label{eq:masf_nonlinear_guidance_score}
    g(\mathbf{z},\mathbf{x},t)
    &:=
    \nabla_{\mathbf{x}}
    \log
    \mathcal{N}
    \left(
        \mathbf{z};
        \mu_t(\mathbf{x}),
        \Sigma_{t\to 1}
    \right)
    \\
    &=
    \left(\nabla_{\mathbf{x}}\mu_t(\mathbf{x})\right)^{\mathsf T}
    \Sigma_{t\to 1}^{-1}
    \left(
        \mathbf{z}-\mu_t(\mathbf{x})
    \right).
\end{align}
\end{corollary}

\begin{proof}
This follows by differentiating the Gaussian log-density in
\eqref{eq:masf_endpoint_gaussian_approx} with respect to $\mathbf{x}$ and applying the chain rule through $\mu_t(\mathbf{x})$.
\end{proof}

\subsection{Guided Reverse-Time Sampling}
\label{app:guided_reverse_sampling}

We combine the prior score and the likelihood score using the posterior-score decomposition.
For nonlinear measurements, we use the likelihood-score approximation in \eqref{eq:masf_nonlinear_guidance_score}.
For linear measurements, we use the closed-form likelihood score derived in Section~\ref{app:linear_likelihood_score}.

\begin{lemma}[Guided reverse update]
\label{lem:masf_guided_reverse_update}
Let $s>t$ be reverse-time discretization points and define $\Sigma_{s\to t}:=\Sigma(t)-\Sigma(s)$.
Given $\mathbf{x}_s$, define
\begin{align}
\label{eq:masf_hhat_s_definition}
    \widehat h_s(\mathbf{x}_s)
    :=
    a(s)m_1(\mathbf{x}_s,s)
    +
    \bigl(1-a(s)\bigr)m_2(\mathbf{x}_s,s).
\end{align}
Estimate the prior score by
\begin{align}
\label{eq:masf_prior_score_estimator_m}
    s_{\mathrm{prior}}(\mathbf{x}_s,s)
    :=
    -\Sigma(s)^{-1}
    \left(
        \mathbf{x}_s-\widehat h_s(\mathbf{x}_s)
    \right).
\end{align}
Let
\begin{align}
\label{eq:masf_guided_score}
    s_{\mathrm{guide}}(\mathbf{x}_s,s,\mathbf{z})
    :=
    s_{\mathrm{prior}}(\mathbf{x}_s,s)
    +
    g(\mathbf{z},\mathbf{x}_s,s).
\end{align}
Then the guided reverse update is
\begin{align}
\label{eq:masf_guided_sampler_step}
    \mathbf{x}_t
    =
    \mathbf{x}_s
    +
    \bigl(a(t)-a(s)\bigr)
    \left(
        m_1(\mathbf{x}_s,s)-m_2(\mathbf{x}_s,s)
    \right)
    -
    \Sigma_{s\to t}
    s_{\mathrm{guide}}(\mathbf{x}_s,s,\mathbf{z})
    +
    \sqrt{-\Sigma_{s\to t}}\,\eta,
\end{align}
where $\eta\sim\mathcal{N}(0,I)$.
\end{lemma}

\begin{proof}
The first drift term follows from discretizing the projected forward drift in reverse time:
\begin{align}
    \int_s^t
    \dot a(u)
    \left(
        m_1(\mathbf{x}_s,s)-m_2(\mathbf{x}_s,s)
    \right)du
    \approx
    \bigl(a(t)-a(s)\bigr)
    \left(
        m_1(\mathbf{x}_s,s)-m_2(\mathbf{x}_s,s)
    \right).
\end{align}
The score term is obtained from the posterior-score decomposition
\begin{align}
    \nabla_{\mathbf{x}_s}\log p(\mathbf{x}_s\mid \mathbf{z})
    =
    \nabla_{\mathbf{x}_s}\log p(\mathbf{x}_s)
    +
    \nabla_{\mathbf{x}_s}\log p(\mathbf{z}\mid \mathbf{x}_s),
\end{align}
with the two terms approximated by $s_{\mathrm{prior}}$ and $g$, respectively.
Since $\Sigma(t)$ is increasing in the positive-semidefinite order and $s>t$,
$\Sigma_{s\to t}=\Sigma(t)-\Sigma(s)\preceq 0$.
Thus $-\Sigma_{s\to t}\succeq 0$ defines the reverse-time covariance used both in the score correction and in the diffusion term $\sqrt{-\Sigma_{s\to t}}\,\eta$.
\end{proof}

In practice, we use time-dependent weights to balance the prior and likelihood terms:
\begin{align}
\label{eq:masf_scaled_guided_score}
    s_{\mathrm{guide}}(\mathbf{x}_s,s,\mathbf{z})
    =
    \lambda_s(s)\,s_{\mathrm{prior}}(\mathbf{x}_s,s)
    +
    \lambda_g(s)\,g(\mathbf{z},\mathbf{x}_s,s).
\end{align}

\subsection{Learning Conditional Means}
\label{app:learning_conditional_means}

It remains to estimate the conditional means $m_1$ and $m_2$ appearing in the projected drift and guided sampler.
Although these conditional means are defined through the non-Markovian process $\widetilde X_t$, the explicit forward process $X_t$ induces the same joint law for the pair of initial and perturbed states:
\begin{align}
    X_t
    =
    h(X_0,t)+\sqrt{\Sigma(t)}\,\eta,
    \qquad
    \widetilde X_t
    =
    h(\widetilde X_0,t)+\sqrt{\Sigma(t)}\,\eta,
    \qquad
    \widetilde X_0\stackrel{d}{=}X_0 .
\end{align}
Thus, in practice, the conditional means can be estimated using samples from the forward process in \eqref{eq:masf_forward_process}.
The following proposition shows that the corresponding estimators are obtained by minimizing denoising $L_2$ objectives.

\begin{proposition}[Denoising objective for conditional means]
\label{prop:masf_training_conditional_means}
Let
\begin{align}
    X_t=h(X_0,t)+\sqrt{\Sigma(t)}\,\eta,
    \qquad
    \eta\sim\mathcal{N}(0,I).
\end{align}
Consider neural estimators $m_{1,\theta}$ and $m_{2,\theta}$ trained by
\begin{align}
\label{eq:masf_m1_m2_training_loss}
    \mathcal{L}_{\mathrm{reg}}(\theta)
    =
    \mathbb{E}_{t,X_0,\eta}
    \left[
        \left\|
            m_{1,\theta}(X_t,t)-X_0
        \right\|^2
        +
        \left\|
            m_{2,\theta}(X_t,t)-H(X_0)
        \right\|^2
    \right].
\end{align}
Then the minimizers satisfy
\begin{align}
\label{eq:masf_conditional_mean_optimality}
    m_{1,\theta}^{\star}(\mathbf{x},t)
    =
    \mathbb{E}[X_0\mid X_t=\mathbf{x}],
    \qquad
    m_{2,\theta}^{\star}(\mathbf{x},t)
    =
    \mathbb{E}[H(X_0)\mid X_t=\mathbf{x}].
\end{align}
Equivalently, since $X_t$ and $\widetilde X_t$ admit the same conditional representation with $X_0\stackrel{d}{=}\widetilde X_0$, these minimizers recover the conditional means used in \eqref{eq:masf_m1_m2_definition}.
\end{proposition}

\begin{proof}
This follows from the standard $L^2$ projection identity.
For a fixed $t$, the minimizer of
\begin{align}
    \mathbb{E}\left[\|f(X_t,t)-Y\|^2\right]
\end{align}
is the conditional mean $\mathbb{E}[Y\mid X_t]$.
Applying this identity with $Y=X_0$ gives
\begin{align}
    m_{1,\theta}^{\star}(\mathbf{x},t)
    =
    \mathbb{E}[X_0\mid X_t=\mathbf{x}],
\end{align}
and applying it with $Y=H(X_0)$ gives
\begin{align}
    m_{2,\theta}^{\star}(\mathbf{x},t)
    =
    \mathbb{E}[H(X_0)\mid X_t=\mathbf{x}].
\end{align}
Because the non-Markovian process satisfies
$\widetilde X_t=h(\widetilde X_0,t)+\sqrt{\Sigma(t)}\eta$ with
$\widetilde X_0\stackrel{d}{=}X_0$, the conditional laws induced by
$(X_0,X_t)$ and $(\widetilde X_0,\widetilde X_t)$ coincide.
Therefore, the denoising objectives estimate the conditional means required in the projected drift and guided sampler.
\end{proof}

\subsection{Closed-Form Likelihood Score for Linear Measurements}
\label{app:linear_likelihood_score}
We now consider the linear case, where $H(X_0)=AX_0$ for some matrix $A$.

\begin{proposition}[Conditional means for the linear forward process]
\label{prop:masf_linear_conditional_means}
For the linear measurement-aware forward process
\begin{align}
\label{eq:masf_linear_forward_process}
    X_t
    =
    A(t)X_0
    +
    \sqrt{\Sigma(t)}\,\eta,
    \qquad
    A(t)
    =
    a(t)I+\bigl(1-a(t)\bigr)A,
\end{align}
the conditional means satisfy
\begin{align}
\label{eq:masf_linear_m2_m1}
    m_2(\mathbf{x},t)
    =
    A\,m_1(\mathbf{x},t).
\end{align}
\end{proposition}

\begin{proof}
By linearity of $H$,
\begin{align}
    m_2(\mathbf{x},t)
    =
    \mathbb{E}[H(X_0)\mid X_t=\mathbf{x}]
    =
    \mathbb{E}[AX_0\mid X_t=\mathbf{x}]
    =
    A\,\mathbb{E}[X_0\mid X_t=\mathbf{x}]
    =
    A\,m_1(\mathbf{x},t).
\end{align}
\end{proof}

\begin{proposition}[Score-form projected drift for linear measurements]
\label{prop:masf_linear_markovian_projection_drift}
Assume that $A(t)$ is invertible for $t\in[0,1)$.
Let $p_t$ denote the density of $X_t$, and define
\begin{align}
\label{eq:masf_Ft_definition}
    F_t
    :=
    \dot A(t)A(t)^{-1}.
\end{align}
Then the projected drift can be written as
\begin{align}
\label{eq:masf_linear_projected_drift_score_form}
    b(\mathbf{x},t)
    =
    F_t\mathbf{x}
    +
    F_t\Sigma(t)\nabla_{\mathbf{x}}\log p_t(\mathbf{x}).
\end{align}
\end{proposition}

\begin{proof}
For the linear Gaussian corruption
\begin{align}
    X_t=A(t)X_0+\sqrt{\Sigma(t)}\,\eta,
\end{align}
the denoising identity gives
\begin{align}
\label{eq:masf_linear_denoising_identity}
    A(t)m_1(\mathbf{x},t)
    =
    \mathbf{x}
    +
    \Sigma(t)\nabla_{\mathbf{x}}\log p_t(\mathbf{x}).
\end{align}
Since $m_2=A m_1$, Proposition~\ref{prop:masf_projected_drift} yields
\begin{align}
    b(\mathbf{x},t)
    =
    \dot a(t)(I-A)m_1(\mathbf{x},t).
\end{align}
Using $\dot A(t)=\dot a(t)(I-A)$ and substituting \eqref{eq:masf_linear_denoising_identity}, we obtain
\begin{align}
    b(\mathbf{x},t)
    =
    \dot A(t)A(t)^{-1}
    \left(
        \mathbf{x}
        +
        \Sigma(t)\nabla_{\mathbf{x}}\log p_t(\mathbf{x})
    \right),
\end{align}
which gives \eqref{eq:masf_linear_projected_drift_score_form}.
\end{proof}

\begin{theorem}[Linear Markov realization]
\label{thm:masf_linear_markov_realization}
Assume that $A(t)$ is invertible for $t\in[0,1)$ and let
\begin{align}
    F_t=\dot A(t)A(t)^{-1}.
\end{align}
Then the linear measurement-aware forward process admits a linear Markov realization
\begin{align}
\label{eq:masf_linear_markov_realization_sde}
    d\check X_t
    =
    F_t\check X_t\,dt
    +
    G_t\,dB_t,
\end{align}
where \(G_t\) is chosen such that
\begin{align}
\label{eq:masf_Gt_condition}
    G_tG_t^{\mathsf T}
    =
    \dot\Sigma(t)
    -
    F_t\Sigma(t)
    -
    \Sigma(t) F_t^{\mathsf T},
\end{align}
provided that the right-hand side is positive semidefinite.
Moreover,
\begin{align}
    \mathcal{L}(\check X_t)=\mathcal{L}(X_t),
    \qquad t\in[0,1].
\end{align}
\end{theorem}

\begin{proof}
Since the right-hand side of \eqref{eq:masf_Gt_condition} is positive semidefinite,
there exists a matrix \(G_t\) satisfying \eqref{eq:masf_Gt_condition}.
Then the Fokker--Planck equation of \eqref{eq:masf_linear_markov_realization_sde} is
\begin{align}
    \partial_t \hat p_t
    =
    -
    \nabla\cdot
    \left(
        F_t\mathbf{x}\hat p_t
    \right)
    +
    \frac{1}{2}
    \nabla\cdot
    \left(
        G_tG_t^{\mathsf T}
        \nabla \hat p_t
    \right).
\end{align}
Using \eqref{eq:masf_Gt_condition}, the diffusion term can be rewritten so that
the equation coincides with the Fokker--Planck equation generated by the
score-form projected drift in Proposition~\ref{prop:masf_linear_markovian_projection_drift}.
Since both processes are initialized with the same law, uniqueness of the
Fokker--Planck equation implies that their marginal distributions agree.
\end{proof}

\begin{theorem}[Closed-form likelihood score for linear measurements]
\label{thm:masf_linear_likelihood_score}
Assume that $A(t)$ is invertible for $t\in[0,1)$ and that $\check\Sigma_{t\to 1}$ is positive definite.
Then the linear Markov realization $\check X_t$ in Theorem~\ref{thm:masf_linear_markov_realization} admits the Gaussian endpoint transition
\begin{align}
\label{eq:masf_linear_transition_endpoint}
    Z\mid \check X_t
    \sim
    \mathcal{N}
    \left(
        M_{t\to 1}\check X_t,
        \check \Sigma_{t\to 1}
    \right),
\end{align}
where
\begin{align}
\label{eq:masf_linear_transition_params}
    M_{t\to 1}
    =
    A(1)A(t)^{-1},
    \qquad
    \check \Sigma_{t\to 1}
    =
     \Sigma(1)
    -
    M_{t\to 1} \Sigma(t) M_{t\to 1}^{\mathsf T}.
\end{align}
Consequently,
\begin{align}
\label{eq:masf_linear_likelihood_score}
    \nabla_{\mathbf{x}_t}\log p(\mathbf{z}\mid \mathbf{x}_t)
    =
    M_{t\to 1}^{\mathsf T}
    \check\Sigma_{t\to 1}^{-1}
    \left(
        \mathbf{z}
        -
        M_{t\to 1}\mathbf{x}_t
    \right).
\end{align}
\end{theorem}

\begin{proof}
For the linear Markov realization, the transition from $t$ to $1$ is Gaussian because the SDE is linear with additive Gaussian noise.
The deterministic part propagates as
\begin{align}
    M_{t\to 1}=A(1)A(t)^{-1},
\end{align}
and the covariance is the difference between the endpoint covariance and the propagated current covariance:
\begin{align}
    \check \Sigma_{t\to 1}
    =
     \Sigma(1)
    -
    M_{t\to 1}  \Sigma(t) M_{t\to 1}^{\mathsf T}.
\end{align}
Therefore,
\begin{align}
    Z\mid \check X_t=\mathbf{x}_t
    \sim
    \mathcal{N}
    \left(
        M_{t\to 1}\mathbf{x}_t,
        \check \Sigma_{t\to 1}
    \right).
\end{align}
Differentiating the Gaussian log-density with respect to $\mathbf{x}_t$ gives the stated likelihood score.
\end{proof}

\begin{remark}[Invertibility of $A(t)$]
\label{rem:masf_At_invertibility}
Let $\lambda_i$ be the eigenvalues of $A$.
Since
\begin{align}
    A(t)=a(t)I+\bigl(1-a(t)\bigr)A,
\end{align}
the eigenvalues of $A(t)$ are
\begin{align}
    a(t)+\bigl(1-a(t)\bigr)\lambda_i.
\end{align}
Thus, $A(t)$ is singular only if
\begin{align}
    a(t)
    =
    \frac{-\lambda_i}{1-\lambda_i}
\end{align}
for some eigenvalue $\lambda_i\neq 1$.
In particular, if $\sigma(A)\subset[0,1]$, then $A(t)$ is invertible for all $t\in[0,1)$.
This condition covers common linear measurement operators such as grid masks and center masks.
\end{remark}

\section{Algorithmic Details}
\label{app:algorithmic_details}

This section describes the practical data assimilation procedure used by MASF.
We describe dataset construction, online training at each assimilation time, measurement update, time update, and evaluation metrics.

\subsection{MASF Procedure}
\label{app:masf_procedure}

\paragraph{Dataset construction.}
Let $\{\tau_k\}_{k=0}^{K}$ denote the assimilation time grid.
For filtering, we sample an initial ensemble from the prior distribution,
\begin{align}
    \{\hat{\mathbf{x}}^{(i)}_{\mathrm{init}}\}_{i=1}^{N}
    \sim p_{\texttt{init}},
\end{align}
and propagate each particle to the first assimilation time:
\begin{align}
    \hat{\mathbf{x}}^{(i)}_{0}
    =
    \Phi_{\tau_0}
    \left(
        \hat{\mathbf{x}}^{(i)}_{\mathrm{init}}
    \right),
    \qquad i=1,\ldots,N.
\end{align}
Here, $\Phi_{\tau_0}$ denotes the numerical flow map to $\tau_0$, implemented by repeated applications of the dataset-specific transition solver.

For evaluation, we independently sample a ground-truth initial state
\begin{align}
    \mathbf{x}^{\mathrm{gt}}_{\mathrm{init}}\sim p_{\texttt{init}}
\end{align}
and propagate it along the assimilation grid:
\begin{align}
    \mathbf{x}^{\mathrm{gt}}_{k}
    =
    \Phi_{\tau_k}
    \left(
        \mathbf{x}^{\mathrm{gt}}_{\mathrm{init}}
    \right),
    \qquad k=0,\ldots,K.
\end{align}
The measurement at time $\tau_k$ is generated as
\begin{align}
    \mathbf{z}_{k}
    =
    H(\mathbf{x}^{\mathrm{gt}}_{k})
    +
    \sigma\boldsymbol{\epsilon}_{k},
    \qquad
    \boldsymbol{\epsilon}_{k}\sim\mathcal{N}(0,I),
    \qquad k=0,\ldots,K.
\end{align}
During filtering, MASF uses only the measurement sequence $\{\mathbf{z}_k\}_{k=0}^{K}$.
The ground-truth trajectory is used only for evaluation.

\paragraph{Online training at an assimilation time.}
At each assimilation time $\tau_k$, MASF trains the conditional-mean estimators using the current prior ensemble.
Given prior particles $\{\hat{\mathbf{x}}^{(i)}_{k}\}_{i=1}^{N}$, we generate perturbed samples through the measurement-aware forward process:
\begin{align}
    \hat{\mathbf{x}}^{(i)}_{k,t}
    =
    h(\hat{\mathbf{x}}^{(i)}_{k},t)
    +
    \sqrt{\Sigma(t)}\,\boldsymbol{\epsilon}^{(i)},
    \qquad
    \boldsymbol{\epsilon}^{(i)}\sim\mathcal{N}(0,I).
\end{align}
For nonlinear measurements, we train $m_{1,\theta}$ and $m_{2,\theta}$ by minimizing the denoising $L_2$ objective in \eqref{eq:masf_m1_m2_training_loss}.
For linear measurements $H(\mathbf{x})=A\mathbf{x}$, we train only $m_{1,\theta}$, since $m_{2,\theta}=A m_{1,\theta}$ is available analytically.

To reduce online training cost, the model is initialized from a pretrained checkpoint.
The pretrained model is obtained before filtering using samples generated from the state dynamics, without online measurements.
At the first assimilation time, the model is initialized from the pretrained checkpoint, and at later assimilation times it is fine-tuned on the current prior ensemble.

\paragraph{Measurement update.}
After training at $\tau_k$, MASF performs the measurement update using guided reverse-time sampling.
Starting from the prior ensemble $\{\hat{\mathbf{x}}^{(i)}_{k}\}_{i=1}^{N}$, we set $s_0=0.992$ and initialize $\mathbf{x}_{s_0}^{(i)}=\hat{\mathbf{x}}_k^{(i)}$.
Given the measurement $\mathbf{z}_k$, the sampler produces posterior particles:
\begin{align}
    \mathbf{x}^{(i)}_{k}
    =
    \texttt{GuidedSampler}
    \left(
        \mathbf{x}_{s_0}^{(i)},
        \mathbf{z}_k;
        \theta_k
    \right),
    \qquad i=1,\ldots,N.
\end{align}
The sampler combines the prior score with the likelihood score induced by $\mathbf{z}_k$.
For nonlinear measurements, the likelihood score is computed using \eqref{eq:masf_nonlinear_guidance_score}.
For linear measurements, we use the closed-form likelihood score in \eqref{eq:masf_linear_likelihood_score}.
The reverse sampler uses NFE function evaluations.

\paragraph{Time update.}
After assimilating the measurement at $\tau_k$, MASF propagates the posterior ensemble to the next assimilation time.
Let $n_k$ denote the number of numerical transition steps between $\tau_k$ and $\tau_{k+1}$.
The next prior ensemble is
\begin{align}
    \hat{\mathbf{x}}^{(i)}_{k+1}
    =
    \Phi_{n_k}
    \left(
        \mathbf{x}^{(i)}_{k}
    \right),
    \qquad i=1,\ldots,N,
\end{align}
where $\Phi_{n_k}$ denotes $n_k$ repeated applications of the transition solver.

The full procedure is summarized in Algorithm~\ref{alg:measurement_aware_guided_assimilation}.
Additional implementation parameters, including the terminal offset $\varepsilon$, covariance floor, and pseudo-inverse option, are reported in Appendix~\ref{app:noise_schedule_guidance_weights}.

\begin{algorithm}[!htb]
\caption{Measurement-Aware Score-based Filter}
\label{alg:measurement_aware_guided_assimilation}
\small
\begin{algorithmic}[1]
\State \textbf{Input:} measurements $(\mathbf{z}_k)_{k=0}^{K}$, ensemble size $N$, epochs $E$, NFE
\State \textbf{Input:} schedule $a(t)$, variance $\Sigma(t)$, measurement operator $H$, transition solver $\Phi$
\State \textbf{Output:} ensemble-mean estimates $(\bar{\mathbf{x}}_k)_{k=0}^{K}$

\State Sample $(\hat{\mathbf{x}}_{\mathrm{init}}^{(i)})_{i=1}^{N}\sim p_{\texttt{init}}$
\State $\hat{\mathbf{x}}^{(i)}_0 \gets \Phi_{\tau_0}(\hat{\mathbf{x}}_{\mathrm{init}}^{(i)}),\quad i=1,\ldots,N$

\For{$k=0$ \textbf{to} $K$}
  \State Initialize $\theta_k$ from the pretrained or previous-step weights

  \For{$e=1$ \textbf{to} $E$}
    \State Sample $t\sim\mathcal{U}(0,1)$ and $\boldsymbol{\epsilon}^{(i)}\sim\mathcal{N}(0,I)$
    \State $\hat{\mathbf{x}}^{(i)}_{k,t}\gets h(\hat{\mathbf{x}}^{(i)}_k,t)+\sqrt{\Sigma(t)}\,\boldsymbol{\epsilon}^{(i)}$

    \If{$H$ is linear}
      \State $L\gets \frac{1}{N}\sum_{i=1}^{N}
      \|m_{1,\theta_k}(\hat{\mathbf{x}}^{(i)}_{k,t},t)-\hat{\mathbf{x}}^{(i)}_k\|^2$
    \Else
      \State $L\gets \frac{1}{N}\sum_{i=1}^{N}\bigl[
      \|m_{1,\theta_k}(\hat{\mathbf{x}}^{(i)}_{k,t},t)-\hat{\mathbf{x}}^{(i)}_k\|^2
      +
      \|m_{2,\theta_k}(\hat{\mathbf{x}}^{(i)}_{k,t},t)-H(\hat{\mathbf{x}}^{(i)}_k)\|^2
      \bigr]$
    \EndIf

    \State Update $\theta_k$ by minimizing $L$
  \EndFor

  \State Initialize reverse sampling with $s_0=0.992$ and $\mathbf{x}_{s_0}^{(i)}\gets \hat{\mathbf{x}}^{(i)}_k$
  \State $(\mathbf{x}^{(i)}_k)_{i=1}^{N}\gets
  \texttt{GuidedSampler}((\mathbf{x}_{s_0}^{(i)})_{i=1}^{N},\mathbf{z}_k,\theta_k,H,a,\Sigma,\text{NFE})$

  \State $\bar{\mathbf{x}}_k\gets \frac{1}{N}\sum_{i=1}^{N}\mathbf{x}^{(i)}_k$

  \If{$k<K$}
    \State Let $n_k$ be the number of transition steps from $\tau_k$ to $\tau_{k+1}$
    \State $\hat{\mathbf{x}}^{(i)}_{k+1}\gets \Phi_{n_k}(\mathbf{x}^{(i)}_k),\quad i=1,\ldots,N$
  \EndIf
\EndFor
\end{algorithmic}
\end{algorithm}

\begin{algorithm}[!htb]
\caption{\texttt{GuidedSampler}}
\label{alg:guided_sampler}
\small
\begin{algorithmic}[1]
\State \textbf{Input:} particles $(\mathbf{x}_{s_0}^{(i)})_{i=1}^{N}$, measurement $\mathbf{z}_k$, parameters $\theta_k$, measurement operator $H$, schedules $a,\Sigma$, NFE
\State \textbf{Output:} posterior particles $(\mathbf{x}_{k}^{(i)})_{i=1}^{N}$
\State Construct a reverse-time grid $0.992=s_0>s_1>\cdots>s_L=10^{-5}$ with $L=\text{NFE}$

\For{$\ell=0$ \textbf{to} $L-1$}
  \State Set $s\gets s_\ell$ and $t\gets s_{\ell+1}$
  \State Compute $m_1(\mathbf{x}_s^{(i)},s)$ and $m_2(\mathbf{x}_s^{(i)},s)$ using $\theta_k$
  \State Compute $s_{\mathrm{prior}}(\mathbf{x}_s^{(i)},s)$ using \eqref{eq:prior_score_estimator_m}
  \State Compute $g(\mathbf{z}_k,\mathbf{x}_s^{(i)},s)$ using \eqref{eq:nonlinear_guidance_score} or \eqref{eq:linear_likelihood_score}
  \State Form $s_{\mathrm{guide}}(\mathbf{x}_s^{(i)},s,\mathbf{z}_k)$ using \eqref{eq:scaled_guided_score}
  \State Update $\mathbf{x}_t^{(i)}$ using \eqref{eq:guided_sampler_step}, for $i=1,\ldots,N$
\EndFor

\State \textbf{return} $\mathbf{x}_{k}^{(i)}\gets \mathbf{x}_{s_L}^{(i)},\quad i=1,\ldots,N$
\end{algorithmic}
\end{algorithm}

\subsection{Evaluation Metrics}
\label{app:evaluation_metrics}

At each assimilation time, the state estimate is the ensemble mean:
\begin{align}
    \bar{\mathbf{x}}_k
    =
    \frac{1}{N}
    \sum_{i=1}^{N}
    \mathbf{x}^{(i)}_k .
\end{align}
Given the ground-truth trajectory $\{\mathbf{x}^{\mathrm{gt}}_k\}_{k=0}^{K}$, we report the root mean squared error (RMSE) averaged over the full assimilation window:
\begin{align}
    \mathrm{RMSE}
    =
    \sqrt{
    \frac{1}{(K+1)d}
    \sum_{k=0}^{K}
    \left\|
        \bar{\mathbf{x}}_k
        -
        \mathbf{x}^{\mathrm{gt}}_k
    \right\|_2^2
    } .
\end{align}
We also report time-wise RMSE by applying the same metric separately at each assimilation time.

For spatial field experiments, we additionally report the critical success index (CSI), which measures the overlap between predicted and ground-truth high-intensity regions:
\begin{align}
    \mathrm{CSI}
    =
    \frac{\mathrm{TP}}
    {\mathrm{TP}+\mathrm{FP}+\mathrm{FN}}.
\end{align}
Here, $\mathrm{TP}$, $\mathrm{FP}$, and $\mathrm{FN}$ denote true positives, false positives, and false negatives after thresholding the predicted and ground-truth fields.
Unless a fixed threshold is specified, the event threshold is chosen as the $0.95$ quantile of the ground-truth values.

\section{Implementation and Tuning Details}
\label{app:implementation_tuning}

This section provides implementation details and hyperparameter tuning procedures.
We describe the efficient implementation of MASF, the sequential tuning protocol, the final selected configurations, and the tuning ranges for baselines.

\subsection{Compute Resources}
\label{app:compute_resources}

Table~\ref{tab:compute_resources} summarizes the compute environment used in our experiments.
Wall-clock times for the online data assimilation procedure are reported in the experimental tables.
For MASF, the reported wall-clock time excludes offline pretraining and includes only online fine-tuning and guided sampling.
Unless otherwise specified, each neural-network-based run used one GPU.
All quantitative results are averaged over five random seeds.

\begin{table}[!htbp]
\centering
\caption{Compute resources and software environment used for the experiments.}
\label{tab:compute_resources}
\small
\begin{tabular}{ll}
\toprule
Resource & Specification \\
\midrule
Operating system & Linux, kernel 5.14.0 \\
CPU & Intel Xeon 6530P, 2 sockets \\
CPU capacity & 64 physical cores, 128 logical CPUs total \\
CPU allocation & 8 logical CPUs per Slurm task \\
System memory & 250 GiB RAM, 63 GiB swap \\
GPU & NVIDIA RTX PRO 6000 Blackwell Server Edition \\
GPU memory & 97,887 MiB per GPU \\
NVIDIA driver & 580.105.08 \\
CUDA & 13.0, as reported by \texttt{nvidia-smi} \\
PyTorch & 2.9.0 with CUDA 13.0 support \\
JAX / JAX-CFD & JAX 0.4.33, JAX-CFD 0.2.1 \\
Numerical packages & NumPy 2.1.1, SciPy 1.14.1 \\
GPU usage & One GPU per neural-network-based run unless otherwise specified \\
Random seeds & 5 seeds: 42, 43, 44, 45, 46 \\
Reported runtime & Online data assimilation wall-clock time \\
MASF runtime convention & Offline pretraining excluded from reported online wall-clock time \\
\bottomrule
\end{tabular}
\end{table}

\subsection{Existing Assets, Versions, and Licenses}
\label{app:existing_assets}

Table~\ref{tab:existing_assets} summarizes the main existing software and methodological assets used in the experiments.
The Kolmogorov flow data used in this paper are generated by the authors through simulation and are not repackaged from an external dataset.

\begin{table}[!htbp]
\centering
\caption{Existing assets used in the experiments.}
\label{tab:existing_assets}
\small
\begin{tabular}{llll}
\toprule
Asset & Use & Version / source & License / terms \\
\midrule
JAX-CFD & Kolmogorov flow simulation & 0.2.1 & Apache-2.0 \\
JAX & Numerical computation & 0.4.33 & Apache-2.0 \\
PyTorch & Neural network training & 2.9.0 & BSD-style license \\
NumPy & Numerical computation & 2.1.1 & BSD 3-Clause \\
SciPy & Numerical computation & 1.14.1 & BSD 3-Clause \\
Optuna & Hyperparameter tuning & Cited package & MIT license \\
U-Net & Architecture reference & Cited paper & Reimplemented by authors \\
EnKF / LETKF & Baseline methods & Cited papers & Reimplemented by authors \\
SF / SSLS & Score-based baselines & Cited papers & Reimplemented by authors \\
\bottomrule
\end{tabular}
\end{table}



\subsection{Noise Schedule and Guidance Weights}
\label{app:noise_schedule_guidance_weights}

We use the cosine schedule
\begin{align}
    a(t)
    =
    \frac{
    \cos\left( \frac{t+s_{\mathrm{cos}}}{1+s_{\mathrm{cos}}}\frac{\pi}{2} \right)
    }{
    \cos\left( \frac{s_{\mathrm{cos}}}{1+s_{\mathrm{cos}}}\frac{\pi}{2} \right)
    },
    \qquad
    r(t)=\sqrt{1-a^2(t)},
\end{align}
with $s_{\mathrm{cos}}=0.008$.
Equivalently, the reverse-time sampler starts from $1-\varepsilon=0.992$ with $\varepsilon=0.008$ and terminates at $10^{-5}$; thus, likelihood scores are never evaluated at $t=1$.

For the guided score in \eqref{eq:masf_scaled_guided_score}, we use time-dependent weights
\begin{align}
    \lambda_s(s)
    &=
    \lambda_s^{\min}
    +
    \left(
    \lambda_s^{\max}
    -
    \lambda_s^{\min}
    \right)s^{p_s},
    \\
    \lambda_g(s)
    &=
    \lambda_g^{\min}
    +
    \left(
    \lambda_g^{\max}
    -
    \lambda_g^{\min}
    \right)(1-s)^{p_g}.
\end{align}
In all experiments, we set
\begin{align}
    \lambda_s^{\min}=1,\quad
    \lambda_s^{\max}=1.5,\quad
    p_s=1,
    \qquad
    \lambda_g^{\min}=1.1,\quad
    \lambda_g^{\max}=1.5,\quad
    p_g=2.
\end{align}
For covariance inversion, we add a small diagonal floor and use
\begin{align}
    \Sigma^{-1}
    \leftarrow
    \left(\Sigma+\delta I\right)^{-1},
\end{align}
with $\delta=10^{-5}$. We do not use a pseudo-inverse unless otherwise specified.

\subsection{Efficient Implementation}
\label{app:efficient_implementation}

MASF is designed to reduce the online computational cost of score-based data assimilation.
The main runtime bottlenecks of score-based filtering are online neural-network training and reverse-time sampling.
We reduce these costs using three implementation strategies: a lightweight conditional-mean architecture, amortized pretraining, and a reduced NFE.

\paragraph{Lightweight architecture.}
Unlike unconditional score-based generative modeling, data assimilation does not require learning a global data distribution for open-ended generation.
At each assimilation step, the model is trained only on the current prior ensemble and is used to estimate the conditional means required by the projected drift and guided sampler.
Thus, the learning problem is localized around the forecast distribution at the current assimilation time.
This allows us to use a lightweight U-Net rather than a large diffusion backbone.
For linear measurements, we further reduce the output dimension by training only $m_{1,\theta}$, since $m_2=A m_1$ is available analytically.
For nonlinear measurements, we use a dual-head U-Net with a shared encoder and bottleneck and separate decoders for $m_{1,\theta}$ and $m_{2,\theta}$.
This keeps most feature extraction shared while allowing the two conditional means to have different outputs.

\paragraph{Amortized pretraining.}
Training a neural score or denoising model from scratch at every assimilation time is expensive.
To reduce this cost, we pretrain the conditional-mean network before filtering using samples generated from the state dynamics, without using online measurements.
During filtering, the pretrained weights are used as initialization and are fine-tuned on the current prior ensemble at each measurement update.
At the first assimilation time, this avoids cold-start training; at later assimilation times, the previous or pretrained weights provide a warm start for adaptation to the updated forecast distribution.
Because consecutive assimilation steps have related forecast distributions, this amortized initialization substantially reduces the number of online optimization steps needed for stable conditional-mean estimation.
The offline pretraining cost is incurred once per configuration and is excluded from the reported online wall-clock time.

\paragraph{Reduced NFE.}
Score-based filtering methods can be expensive at inference time because reverse-time sampling typically requires many function evaluations.
MASF reduces this cost by using measurement-aware likelihood guidance, which directly steers particles toward the measurement-conditioned posterior.
As a result, the sampler requires fewer reverse-time discretization steps than score-based baselines that rely on a measurement-independent forward process.
In our experiments, MASF uses $\mathrm{NFE}=150$--$250$ depending on the resolution and measurement setting, whereas SF and SSLS are run with $\mathrm{NFE}=500$.
The selected NFE values are chosen through the sequential tuning procedure in Appendix~\ref{app:masf_tuning}.
This reduction directly lowers sampling time, while the measurement-aware guidance preserves filtering accuracy.

\subsection{MASF Sequential Tuning Procedure}
\label{app:masf_tuning}

We tuned MASF sequentially rather than jointly over all hyperparameters.
The tuning procedure was automated in code using Optuna~\cite{akiba2019optuna}, with each phase run for a fixed number of trials specified by the corresponding \texttt{n\_trials} setting.
For each trial, Optuna sampled a candidate configuration from the predefined search space, and completed trials were ranked by the final filtering metric.
After each phase, the best-performing configuration was used as the base configuration for the next phase.
The tuning phases were normalization selection, model selection, pretraining, sampling and guidance, and online fine-tuning.
Unless otherwise stated, each phase used $N=10$ particles, pretrained initialization, and the assimilation window from step $50$ to step $100$ with measurement gap $10$.
Thus, configurations were selected based on final filtering performance rather than training loss alone.

\paragraph{Normalization selection.}
We first tune the normalization module because the measurement-aware forward process interpolates between state and measurement variables.
If these variables have different scales or statistics, the forward process can become poorly conditioned and the conditional-mean estimators may become unstable.
Given data $x\in\mathbb{R}^{B\times C\times\cdots}$, normalization statistics are computed per channel over all non-channel dimensions.

\begin{table}[!htbp]
\centering
\caption{Normalization search space for MASF.}
\label{tab:normalization_tuning}
\small
\begin{tabular}{ll}
\toprule
Parameter & Search values \\
\midrule
same normalization & \{\texttt{true}, \texttt{false}\} \\
normalization form & \{\texttt{affine}, \texttt{scale\_only}\} \\
statistics mode & \{\texttt{standard}, \texttt{robust}, \texttt{minmax}\} \\
statistics update & \{\texttt{moving}, \texttt{fixed}, \texttt{adaptive}\} \\
momentum & \{0.2, 0.4\}\\
\bottomrule
\end{tabular}
\end{table}

The parameter \textit{same normalization} determines whether the measurement variable shares normalization statistics with the state variable.
When it is set to \texttt{true}, both $X$ and $Z$ are normalized using statistics computed from the state ensemble.
When it is set to \texttt{false}, separate measurement statistics are computed after applying the measurement operator to the state ensemble.
The \textit{normalization form} determines whether variables are centered and scaled or only rescaled:
\begin{align}
    x_{\mathrm{affine}}
    =
    \frac{x-u}{s+\epsilon},
    \qquad
    x_{\mathrm{scale}}
    =
    \frac{x}{s+\epsilon},
\end{align}
where $u$ and $s$ are the per-channel center and scale.
The \textit{statistics mode} determines how $u$ and $s$ are computed.
The \texttt{standard} mode uses the mean and standard deviation, the \texttt{robust} mode uses the median and median absolute deviation, and the \texttt{minmax} mode uses the midpoint and half-range of the channel values.
The \textit{statistics update} determines how normalization statistics evolve over assimilation steps.
The \texttt{fixed} mode keeps the initial statistics fixed, the \texttt{adaptive} mode recomputes them at each update, and the \texttt{moving} mode uses an exponential moving average:
\begin{align}
    \theta_{\mathrm{new}}
    =
    (1-\rho)\theta_{\mathrm{old}}
    +
    \rho\theta_{\mathrm{current}},
\end{align}
where $\rho$ is the momentum.

\paragraph{Model selection.}
After fixing the normalization scheme, we tune the U-Net architecture.
We search over width, depth, channel-multiplier schedule, and attention resolution.
Because the computational budget and spatial resolution differ across Kolmogorov-$128$, Kolmogorov-$256$, and Kolmogorov-$512$, we use resolution-specific search spaces.

\begin{table}[!htbp]
\centering
\caption{Model architecture search spaces for MASF on Kolmogorov datasets.}
\label{tab:model_tuning}
\small
\begin{tabular}{llll}
\toprule
Parameter & Kolmogorov-$128$ & Kolmogorov-$256$ & Kolmogorov-$512$ \\
\midrule
base channels 
& \{16, 32\} 
& \{16, 32\} 
& \{16, 32\} \\

residual blocks 
& \{1, 2, 3\} 
& \{1, 2, 3\} 
& \{1, 2, 3\} \\

channel multiplier 
& \{\texttt{1-2-4}, \texttt{1-2-4-8}\} 
& \{\texttt{1-2-4-8}, \texttt{1-1-2-3-4}\} 
& \{\texttt{1-2-4-8}, \texttt{1-1-2-3-4}\} \\

attention resolution 
& \{8, 16\} 
& \{8, 16, 32\} 
& \{8, 16, 32\} \\
\bottomrule
\end{tabular}
\end{table}

The \textit{base channels} parameter controls the width of the U-Net.
The \textit{residual blocks} parameter determines the number of residual blocks at each resolution.
The \textit{channel multiplier} specifies how the number of channels changes across resolutions.
The \textit{attention resolution} determines the spatial resolution at which self-attention is applied.
The selected architectures are measurement- and resolution-dependent and are reported in Tables~\ref{tab:final_masf_grid_config} and~\ref{tab:final_masf_measurement_config}.

\paragraph{Pretraining.}
We next tune pretraining hyperparameters while fixing the number of pretraining samples to $1000$.
The batch size, number of epochs, and learning rate control the optimization budget.
We use Adam with weight decay $0.01$ and betas $(0.9,0.99)$.
The selected pretraining configurations are measurement- and resolution-dependent and are reported in Tables~\ref{tab:final_masf_grid_config} and~\ref{tab:final_masf_measurement_config}.

\begin{table}[!htbp]
\centering
\caption{Pretraining search space for MASF.}
\label{tab:pretraining_tuning}
\small
\begin{tabular}{ll}
\toprule
Parameter & Search values \\
\midrule
pretraining samples & \{1000\} \\
batch size & \{16, 32, 64\} \\
epochs & \{100, 200, 300, 400\} \\
learning rate & \{$0.0001$, $0.0003$, $0.0005$\} \\
\bottomrule
\end{tabular}
\end{table}

\paragraph{Sampling and guidance.}
We tune the sampling configuration to reduce wall-clock time while preserving filtering accuracy.
In the guided reverse sampler, we use time-dependent scale factors for the prior score and likelihood score:
\begin{align}
\label{eq:sample_score_scale}
    \lambda_s(t)
    &=
    \lambda_{s,\min}
    +
    \left(
        \lambda_{s,\max}
        -
        \lambda_{s,\min}
    \right)t^{p_s}, \\
\label{eq:sample_guidance_scale}
    \lambda_g(t)
    &=
    \lambda_{g,\min}
    +
    \left(
        \lambda_{g,\max}
        -
        \lambda_{g,\min}
    \right)(1-t)^{p_g}.
\end{align}
Here, $\lambda_s(t)$ scales the learned prior score, while $\lambda_g(t)$ scales the likelihood score.
The prior scale increases with $t$, giving stronger prior regularization near the noisy endpoint.
The likelihood scale increases as $t$ moves toward $0$, strengthening measurement correction during the later part of reverse sampling.
The powers $p_s$ and $p_g$ control how sharply these scale factors change over time.

\begin{table}[!htbp]
\centering
\caption{Sampling and guidance search space for MASF.}
\label{tab:masf_sampling_tuning}
\small
\begin{tabular}{ll}
\toprule
Parameter & Search values \\
\midrule
NFE & \{150, 200, 250\}\\
prior scale min. $\lambda_{s,\min}$ & \{0.9, 0.95, 1.0\} \\
prior scale max. $\lambda_{s,\max}$ & \{1.4, 1.5, 1.6\} \\
prior scale power $p_s$ & \{1.0, 1.25, 1.5\} \\
likelihood scale min. $\lambda_{g,\min}$ & \{1.0, 1.1, 1.2\} \\
likelihood scale max. $\lambda_{g,\max}$ & \{1.4, 1.5, 1.6\} \\
likelihood scale power $p_g$ & \{1.5, 2.0\} \\
\bottomrule
\end{tabular}
\end{table}

The \textit{NFE} controls the number of reverse-time discretization steps.
The parameters $\lambda_{s,\min}$ and $\lambda_{s,\max}$ determine the range of the prior-score scale, and $p_s$ controls its time dependence.
Similarly, $\lambda_{g,\min}$ and $\lambda_{g,\max}$ determine the range of the likelihood-score scale, and $p_g$ controls how strongly the likelihood score is emphasized along the reverse trajectory.

\paragraph{Online fine-tuning.}
Finally, we tune the online fine-tuning parameters used during assimilation.
Starting from the pretrained model, the network is fine-tuned on the current prior ensemble at each measurement step.
The search space is summarized in Table~\ref{tab:masf_finetuning_tuning}.
During this phase, we fix the ensemble size to $N=10$, use pretrained initialization, and evaluate the assimilation window from step $50$ to step $100$ with measurement gap $10$.

\begin{table}[!htbp]
\centering
\caption{Online fine-tuning search space for MASF.}
\label{tab:masf_finetuning_tuning}
\small
\begin{tabular}{ll}
\toprule
Parameter & Search values \\
\midrule
batch size & \{10\} \\
learning rate & \{$3\times10^{-4}$, $1\times10^{-4}$\} \\
full epochs & \{80, 90, 100, 110, 120\} \\
\bottomrule
\end{tabular}
\end{table}

\subsection{Final MASF Configuration}
\label{app:final_masf_config}

The selected MASF hyperparameters for the main Kolmogorov experiments are summarized in
Tables~\ref{tab:final_masf_grid_config} and~\ref{tab:final_masf_measurement_config}.
Table~\ref{tab:final_masf_grid_config} reports the grid-mask configurations across resolutions,
and Table~\ref{tab:final_masf_measurement_config} reports the remaining $128^2$ measurement settings.
We report the parameters that define the selected experimental configurations; omitted parameters are kept at their default values.

\begin{table}[!htbp]
\centering
\caption{Selected MASF configurations for grid-mask measurements across resolutions.}
\label{tab:final_masf_grid_config}
\scriptsize
\setlength{\tabcolsep}{4pt}
\renewcommand{\arraystretch}{1.08}
\resizebox{\textwidth}{!}{%
\begin{tabular}{llccc}
\toprule
Category & Field & Grid mask $128^2$ & Grid mask $256^2$ & Grid mask $512^2$ \\
\midrule
\multirow{7}{*}{measurement}
& type & \texttt{grid\_mask} & \texttt{grid\_mask} & \texttt{grid\_mask} \\
& same normalization & \texttt{true} & \texttt{false} & \texttt{false} \\
& normalization form & \texttt{affine} & \texttt{affine} & \texttt{scale\_only} \\
& statistics mode & \texttt{standard} & \texttt{standard} & \texttt{standard} \\
& statistics update & \texttt{moving} & \texttt{moving} & \texttt{adaptive} \\
& main parameter & stride $=10$ & stride $=15$ & stride $=20$ \\
& momentum & $0.2$ & $0.2$ & -- \\
\midrule
\multirow{6}{*}{model}
& type & U-Net & U-Net & U-Net \\
& input channels & $2$ & $2$ & $2$ \\
& base channels & $32$ & $16$ & $16$ \\
& residual blocks & $2$ & $1$ & $1$ \\
& attention resolution & $16$ & $32$ & $32$ \\
& channel multiplier & $(1,2,4,8)$ & $(1,2,4,8)$ & $(1,2,4,8)$ \\
\midrule
\multirow{3}{*}{pretraining}
& batch size & $16$ & $16$ & $16$ \\
& epochs & $400$ & $300$ & $300$ \\
& learning rate & $0.0003$ & $0.0005$ & $0.0005$ \\
\midrule
\multirow{7}{*}{sampling}
& NFE & $150$ & $200$ & $150$ \\
& prior scale min. & $0.95$ & $0.9$ & $0.95$ \\
& prior scale max. & $1.6$ & $1.5$ & $1.6$ \\
& prior scale power & $1.5$ & $1.5$ & $1.5$ \\
& likelihood scale min. & $1.1$ & $1.0$ & $1.1$ \\
& likelihood scale max. & $1.5$ & $1.6$ & $1.5$ \\
& likelihood scale power & $2.0$ & $1.5$ & $2.0$ \\
\midrule
\multirow{3}{*}{fine-tuning}
& batch size & $10$ & $10$ & $10$ \\
& learning rate & $0.0001$ & $0.0003$ & $0.0003$ \\
& full epochs & $120$ & $110$ & $110$ \\
\bottomrule
\end{tabular}%
}
\end{table}

\begin{table}[!htbp]
\centering
\caption{Selected MASF configurations for the remaining $128^2$ measurement settings.}
\label{tab:final_masf_measurement_config}
\scriptsize
\setlength{\tabcolsep}{4pt}
\renewcommand{\arraystretch}{1.08}
\resizebox{\textwidth}{!}{%
\begin{tabular}{llccc}
\toprule
Category & Field & Center mask & Sigmoid & Speed \\
\midrule
\multirow{7}{*}{measurement}
& type & \texttt{center\_mask} & \texttt{sigmoid} & \texttt{speed} \\
& same normalization & \texttt{false} & \texttt{false} & \texttt{false} \\
& normalization form & \texttt{scale\_only} & \texttt{affine} & \texttt{scale\_only} \\
& statistics mode & \texttt{standard} & \texttt{standard} & \texttt{robust} \\
& statistics update & \texttt{fixed} & \texttt{fixed} & \texttt{adaptive} \\
& main parameter & hole ratio $=0.5$ & $-$ & $-$ \\
\midrule
\multirow{6}{*}{model}
& type & U-Net & dual-head U-Net & dual-head U-Net \\
& input channels & $2$ & $2$ & $2$ \\
& base channels & $32$ & $32$ & $16$ \\
& residual blocks & $2$ & $2$ & $3$ \\
& attention resolution & $16$ & $16$ & $8$ \\
& channel multiplier & $(1,2,4,8)$ & $(1,2,4,8)$ & $(1,2,4,8)$ \\
\midrule
\multirow{3}{*}{pretraining}
& batch size & $16$ & $16$ & $16$ \\
& epochs & $400$ & $200$ & $300$ \\
& learning rate & $0.0003$ & $0.0003$ & $0.0003$ \\
\midrule
\multirow{7}{*}{sampling}
& NFE & $250$ & $150$ & $200$ \\
& prior scale min. & $1.0$ & $0.9$ & $0.9$ \\
& prior scale max. & $1.4$ & $1.4$ & $1.6$ \\
& prior scale power & $1.25$ & $1.25$ & $1.25$ \\
& likelihood scale min. & $1.2$ & $1.2$ & $1.0$ \\
& likelihood scale max. & $1.5$ & $1.6$ & $1.4$ \\
& likelihood scale power & $1.5$ & $2.0$ & $2.0$ \\
\midrule
\multirow{3}{*}{fine-tuning}
& batch size & $10$ & $10$ & $10$ \\
& learning rate & $0.0003$ & $0.0003$ & $0.0003$ \\
& full epochs & $120$ & $120$ & $100$ \\
\bottomrule
\end{tabular}%
}
\end{table}

\subsection{Baseline Tuning Ranges}
\label{app:baseline_tuning}

For fair comparison, we tuned the classical filtering baselines over their main filtering hyperparameters.
For EnKF and LETKF, we searched over ensemble size, inflation, and numerical stabilization parameters.
For LETKF, we additionally tuned the localization radius.
For the score-based baselines, SF and SSLS used the same neural backbone as MASF whenever applicable, while their method-specific sampling parameters were tuned separately.
Both SF and SSLS were run with $\mathrm{NFE}=500$.

\paragraph{EnKF.}
For EnKF, we used the stochastic full-covariance update.
The search space is summarized in Table~\ref{tab:enkf_tuning}.
The ensemble size controls the number of particles.
The inflation factor rescales the forecast covariance to mitigate ensemble underdispersion.
The parameter $\epsilon$ is used for numerical stabilization in covariance operations.

\begin{table}[!htbp]
\centering
\caption{EnKF hyperparameter search space.}
\label{tab:enkf_tuning}
\small
\begin{tabular}{ll}
\toprule
Parameter & Search values \\
\midrule
ensemble size & \{100, 200, 250, 300, 400\} \\
update mode & \{\texttt{full}\} \\
inflation & \{1.0, 1.05, 1.1, 1.15, 1.2, 1.25, 1.3\} \\
epsilon & \{$10^{-5}$, $10^{-6}$\} \\
\bottomrule
\end{tabular}
\end{table}

\paragraph{LETKF.}
For LETKF, we tuned the ensemble size, inflation factor, localization radius, and numerical stabilization parameter.
The search space is summarized in Table~\ref{tab:letkf_tuning}.
The localization radius controls the spatial support of each local analysis update.
The inflation factor rescales the forecast covariance, and $\epsilon$ is used for numerical stabilization.

\begin{table}[!htbp]
\centering
\caption{LETKF hyperparameter search space.}
\label{tab:letkf_tuning}
\small
\begin{tabular}{ll}
\toprule
Parameter & Search values \\
\midrule
ensemble size & \{10, 20, 30, 40, 60, 100\} \\
inflation & \{1.0, 1.05, 1.1, 1.15, 1.2, 1.25, 1.3\} \\
localization radius & \{3, 5, 7\} \\
epsilon & \{$10^{-5}$, $10^{-6}$\} \\
\bottomrule
\end{tabular}
\end{table}

To choose default ensemble sizes for the main experiments, we also evaluated the accuracy--cost trade-off across ensemble sizes under grid-mask measurements.
Table~\ref{tab:ensemble_size_baseline_sensitivity} summarizes the results.
For EnKF, increasing the ensemble size improves RMSE, but larger ensembles substantially increase wall-clock time.
We therefore use $N=250$ as the default EnKF setting, which provides strong accuracy before the cost of larger ensembles becomes high.
For LETKF, the RMSE improvement beyond $N=40$ is modest compared with the additional runtime, so we use $N=40$ as the default LETKF setting.

\begin{table}[!htbp]
\centering
\caption{
Ensemble-size sensitivity for EnKF and LETKF under grid-mask measurements.
The selected default ensemble sizes are highlighted in bold.
}
\label{tab:ensemble_size_baseline_sensitivity}
\small
\setlength{\tabcolsep}{5pt}
\renewcommand{\arraystretch}{1.05}
\begin{tabular}{lccc}
\toprule
Method & Ensemble size & RMSE ($\downarrow$) & Wall-clock (s) \\
\midrule
\multirow{7}{*}{EnKF}
& $40$  & $1.011 \pm 0.069$ & $42.9 \pm 2.2$ \\
& $100$ & $0.540 \pm 0.192$ & $70.6 \pm 1.6$ \\
& $200$ & $0.295 \pm 0.130$ & $139.1 \pm 0.6$ \\
& $\mathbf{250}$ & $\mathbf{0.176 \pm 0.039}$ & $\mathbf{179.5 \pm 2.5}$ \\
& $400$ & $0.141 \pm 0.030$ & $263.5 \pm 9.8$ \\
\midrule
\multirow{6}{*}{LETKF}
& $10$  & $0.222 \pm 0.023$ & $273.1 \pm 4.0$ \\
& $20$  & $0.184 \pm 0.013$ & $294.7 \pm 14.9$ \\
& $30$  & $0.179 \pm 0.014$ & $308.1 \pm 13.5$ \\
& $\mathbf{40}$ & $\mathbf{0.175 \pm 0.014}$ & $\mathbf{389.6 \pm 5.4}$ \\
& $60$  & $0.170 \pm 0.016$ & $455.9 \pm 14.0$ \\
\bottomrule
\end{tabular}
\end{table}

\paragraph{Score-based baselines.}
SF and SSLS used the same neural architecture as MASF whenever applicable.
SF was run with $\mathrm{NFE}=500$.
For SSLS, we tuned the perturbation standard deviation, Langevin stepsize, and tolerance parameter.
The score-based baseline settings are summarized in Table~\ref{tab:score_baseline_tuning}.

\begin{table}[!htbp]
\centering
\caption{Score-based baseline sampling settings.}
\label{tab:score_baseline_tuning}
\small
\begin{tabular}{lll}
\toprule
Method & Parameter & Search values / selected value \\
\midrule
SF, SSLS & NFE & 500 \\
SSLS & perturbation std. & \{0.001, 0.05, 0.1, 0.2\} \\
SSLS & stepsize & \{0.0001, 0.001, 0.002, 0.003, 0.004\} \\
SF, SSLS & tolerance & \{10, 50, 100, 200\} \\
\bottomrule
\end{tabular}
\end{table}

For SSLS, the perturbation standard deviation controls the artificial perturbation scale.
The stepsize controls the Langevin update size, and the tolerance parameter clips or stabilizes large guided-score updates during sampling.

\section{Ablation Studies}
\label{app:ablation_studies}

This section provides ablation studies for MASF.
We first vary individual components of the selected configuration and then examine the effect of pretraining on score-based filtering methods.

\subsection{Component Ablations}
\label{app:ablation}

We conduct component ablations under grid-mask measurements on Kolmogorov-$128$.
Unless otherwise stated, each ablation changes one component from the selected MASF configuration while keeping all other hyperparameters fixed.
The selected configuration uses pretrained initialization, $N=10$ particles, $\mathrm{NFE}=150$, and the assimilation window from step $50$ to step $100$ with measurement gap $10$.
Table~\ref{tab:ablation_study} reports representative ablations.

\begin{table}[htb!]
\centering
\caption{
\textbf{Component ablations.}
Each block reports the current configuration for that phase and one-parameter changes from it.
}
\label{tab:ablation_study}
\scriptsize
\setlength{\tabcolsep}{6pt}
\renewcommand{\arraystretch}{1.03}
\resizebox{0.85\textwidth}{!}{%
\begin{tabular}{lllc}
\toprule
Parameter & From & To & RMSE ($\downarrow$) \\
\midrule

\multicolumn{4}{l}{\textbf{Normalization}} \\
\texttt{current\_config} & -- & --
& $\mathbf{0.151 \pm 0.019}$ \\
\texttt{same\_normalization} & \texttt{true} & \texttt{false}
& $0.229 \pm 0.017$ \\
\texttt{normalization\_form} & \texttt{affine} & \texttt{scale\_only}
& $0.180 \pm 0.013$ \\
\texttt{stats\_mode} & \texttt{standard} & \texttt{robust}
& $0.157 \pm 0.020$ \\
\texttt{stats\_update\_mode} & \texttt{moving} & \texttt{fixed}
& $\mathbf{0.151 \pm 0.018}$ \\
\texttt{stats\_update\_mode} & \texttt{moving} & \texttt{adaptive}
& $0.152 \pm 0.017$ \\
\midrule

\multicolumn{4}{l}{\textbf{Model selection}} \\
\texttt{current\_config} & -- & --
& $0.151 \pm 0.019$ \\
\texttt{model\_channels} & $32$ & $16$
& $0.165 \pm 0.020$ \\
\texttt{num\_res\_blocks} & $2$ & $1$
& $0.167 \pm 0.016$ \\
\texttt{num\_res\_blocks} & $2$ & $3$
& $0.159 \pm 0.016$ \\
\texttt{attention\_resolutions} & $8$ & $16$
& $\mathbf{0.147 \pm 0.015}$ \\
\midrule

\multicolumn{4}{l}{\textbf{Pretraining}} \\
\texttt{current\_config} & -- & --
& $\mathbf{0.146 \pm 0.015}$ \\
\texttt{batch\_size} & $16$ & $32$
& $0.169 \pm 0.021$ \\
\texttt{batch\_size} & $16$ & $64$
& $0.176 \pm 0.018$ \\
\texttt{epoch} & $400$ & $100$
& $0.187 \pm 0.019$ \\
\texttt{epoch} & $400$ & $200$
& $0.165 \pm 0.005$ \\
\texttt{epoch} & $400$ & $300$
& $0.163 \pm 0.018$ \\
\texttt{lr} & $0.0003$ & $0.0001$
& $0.184 \pm 0.024$ \\
\texttt{lr} & $0.0003$ & $0.0005$
& $0.159 \pm 0.021$ \\
\midrule

\multicolumn{4}{l}{\textbf{Sampling}} \\
\texttt{current\_config} & -- & --
& $0.146 \pm 0.015$ \\
\text{NFE} & $200$ & $150$
& $\mathbf{0.144 \pm 0.018}$ \\
\text{NFE} & $200$ & $250$
& $0.147 \pm 0.018$ \\
\texttt{s\_scale\_min} & $0.95$ & $0.9$
& $0.146 \pm 0.015$ \\
\texttt{s\_scale\_min} & $0.95$ & $1.0$
& $0.147 \pm 0.015$ \\
\texttt{s\_scale\_max} & $1.6$ & $1.4$
& $0.150 \pm 0.014$ \\
\texttt{s\_scale\_max} & $1.6$ & $1.5$
& $0.148 \pm 0.014$ \\
\texttt{s\_scale\_power} & $1.5$ & $1.0$
& $0.147 \pm 0.015$ \\
\texttt{s\_scale\_power} & $1.5$ & $1.25$
& $0.146 \pm 0.015$ \\
\texttt{g\_scale\_min} & $1.1$ & $1.0$
& $0.147 \pm 0.013$ \\
\texttt{g\_scale\_min} & $1.1$ & $1.2$
& $0.147 \pm 0.015$ \\
\texttt{g\_scale\_max} & $1.5$ & $1.4$
& $0.146 \pm 0.015$ \\
\texttt{g\_scale\_max} & $1.5$ & $1.6$
& $0.146 \pm 0.015$ \\
\texttt{g\_scale\_power} & $2.0$ & $1.5$
& $0.146 \pm 0.015$ \\
\midrule

\multicolumn{4}{l}{\textbf{Fine-tuning}} \\
\texttt{current\_config} & -- & --
& $\mathbf{0.144 \pm 0.018}$ \\
\texttt{lr} & $0.0001$ & $0.0003$
& $0.154 \pm 0.015$ \\
\texttt{online.full\_epoch} & $120$ & $80$
& $0.151 \pm 0.017$ \\
\texttt{online.full\_epoch} & $120$ & $90$
& $0.150 \pm 0.018$ \\
\texttt{online.full\_epoch} & $120$ & $100$
& $0.151 \pm 0.018$ \\
\texttt{online.full\_epoch} & $120$ & $110$
& $0.148 \pm 0.016$ \\
\bottomrule
\end{tabular}%
}
\end{table}

The ablation results show that normalization, model capacity, and pretraining are the most influential components.
Separate normalization or scale-only normalization degrades RMSE, and reducing model capacity or pretraining budget also worsens RMSE.
In contrast, sampling and online fine-tuning hyperparameters are relatively stable within the tested ranges.
Reducing NFE from $200$ to $150$ preserves RMSE, motivating the efficient configuration used in the main experiments.

\subsection{Effect of Pretraining}
\label{app:pretraining_effect}

We evaluate the effect of pretraining for score-based filtering methods under grid-mask and speed measurements.
For each method and measurement setting, we compare scratch training with $N=100$, scratch training with $N=10$, and pretrained initialization with $N=10$.
Wall-clock time excludes offline pretraining and measures only the online data assimilation procedure.
Each entry in Table~\ref{tab:pretraining_effect} is reported as RMSE / CSI / wall-clock time.

\begin{table}[htb!]
\centering
\caption{
\textbf{Effect of pretraining.}
We compare scratch training and pretrained initialization for score-based filtering methods under grid-mask and speed measurements.
Wall-clock time excludes offline pretraining and measures only the online data assimilation procedure.
}
\label{tab:pretraining_effect}
\scriptsize
\setlength{\tabcolsep}{4pt}
\renewcommand{\arraystretch}{1.05}
\resizebox{\textwidth}{!}{%
\begin{tabular}{lllcccc}
\toprule
Measurement & Method & Strategy & Ensemble & RMSE ($\downarrow$) & CSI ($\uparrow$) & Wall-clock (s) \\
\midrule
\multirow{9}{*}{Grid mask}
& \multirow{3}{*}{SF}
& Scratch & 100 & $0.81 \pm 0.05$ & $0.00 \pm 0.00$ & $670.6 \pm 0.9$ \\
& & Scratch & 10 & $0.87 \pm 0.03$ & $0.00 \pm 0.00$ & $100.9 \pm 0.7$ \\
& & Pretraining & 10 & $0.87 \pm 0.05$ & $0.00 \pm 0.00$ & $81.9 \pm 0.6$ \\
\cmidrule(lr){2-7}
& \multirow{3}{*}{SSLS}
& Scratch & 100 & $0.47 \pm 0.07$ & $0.34 \pm 0.10$ & $611.5 \pm 1.1$ \\
& & Scratch & 10 & $0.58 \pm 0.10$ & $0.26 \pm 0.09$ & $89.2 \pm 0.6$ \\
& & Pretraining & 10 & $0.39 \pm 0.04$ & $0.44 \pm 0.01$ & $81.3 \pm 12.0$ \\
\cmidrule(lr){2-7}
& \multirow{3}{*}{MASF}
& Scratch & 100 & $0.25 \pm 0.04$ & $0.66 \pm 0.04$ & $535.7 \pm 5.2$ \\
& & Scratch & 10 & $0.40 \pm 0.04$ & $0.47 \pm 0.03$ & $139.6 \pm 3.6$ \\
& & Pretraining & 10 & $\mathbf{0.15 \pm 0.02}$ & $\mathbf{0.78 \pm 0.03}$ & $\mathbf{118.8 \pm 2.5}$ \\
\midrule
\multirow{9}{*}{Speed}
& \multirow{3}{*}{SF}
& Scratch & 100 & $0.87 \pm 0.04$ & $0.01 \pm 0.01$ & $380.5 \pm 7.1$ \\
& & Scratch & 10 & $0.95 \pm 0.06$ & $0.00 \pm 0.00$ & $112.4 \pm 5.0$ \\
& & Pretraining & 10 & $0.89 \pm 0.09$ & $0.01 \pm 0.01$ & $86.8 \pm 4.9$ \\
\cmidrule(lr){2-7}
& \multirow{3}{*}{SSLS}
& Scratch & 100 & $0.83 \pm 0.06$ & $0.01 \pm 0.01$ & $358.9 \pm 2.2$ \\
& & Scratch & 10 & $0.87 \pm 0.09$ & $0.03 \pm 0.03$ & $97.1 \pm 0.9$ \\
& & Pretraining & 10 & $0.83 \pm 0.11$ & $0.02 \pm 0.03$ & $75.5 \pm 0.9$ \\
\cmidrule(lr){2-7}
& \multirow{3}{*}{MASF}
& Scratch & 100 & $0.46 \pm 0.06$ & $0.45 \pm 0.06$ & $648.5 \pm 1.4$ \\
& & Scratch & 10 & $0.67 \pm 0.08$ & $0.22 \pm 0.09$ & $171.0 \pm 12.5$ \\
& & Pretraining & 10 & $\mathbf{0.23 \pm 0.04}$ & $\mathbf{0.67 \pm 0.08}$ & $\mathbf{126.2 \pm 5.1}$ \\
\bottomrule
\end{tabular}%
}
\end{table}

Pretraining is most beneficial for MASF in the small-ensemble regime.
With $N=10$, pretraining improves MASF RMSE from $0.40$ to $0.15$ under grid-mask measurements and from $0.67$ to $0.23$ under speed measurements.
The corresponding CSI increases from $0.47$ to $0.78$ and from $0.22$ to $0.67$, respectively.
It also reduces the online wall-clock time compared with scratch training at the same ensemble size, from $139.6$ s to $118.8$ s for grid-mask measurements and from $171.0$ s to $126.2$ s for speed measurements.
In contrast, SF and SSLS show limited accuracy gains from pretraining, indicating that the main benefit comes from combining pretrained initialization with MASF's measurement-aware forward process.

\newpage 
\section{Additional Experimental Results}
\label{app:additional_results}

\subsection{Qualitative Posterior Samples}
\label{app:qualitative_samples}

Figures~\ref{fig:grid_mask}--\ref{fig:speed} show qualitative comparisons under the four measurement settings used in the main experiments.
Each figure presents the ground truth, measurements, and posterior samples from each method at selected time steps.

\begin{figure}[htb!]
    \centering
    \includegraphics[width=0.95\linewidth]{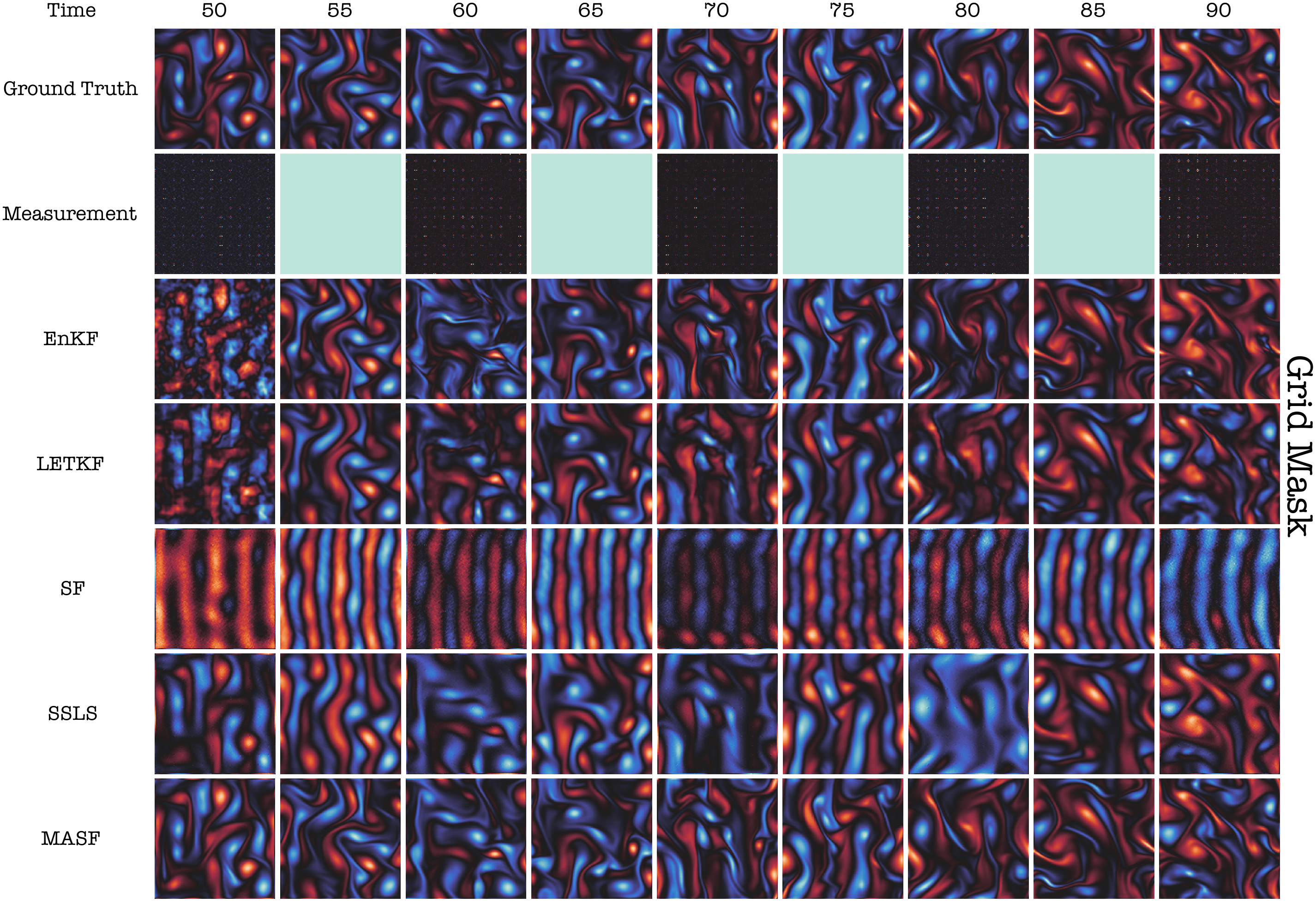}
    \caption{
    \textbf{Grid-mask measurements.}
    Qualitative comparison of measurements and posterior samples under grid-mask measurements at selected time steps.
    }
    \label{fig:grid_mask}
\end{figure}

\begin{figure}[htb!]
    \centering
    \includegraphics[width=0.95\linewidth]{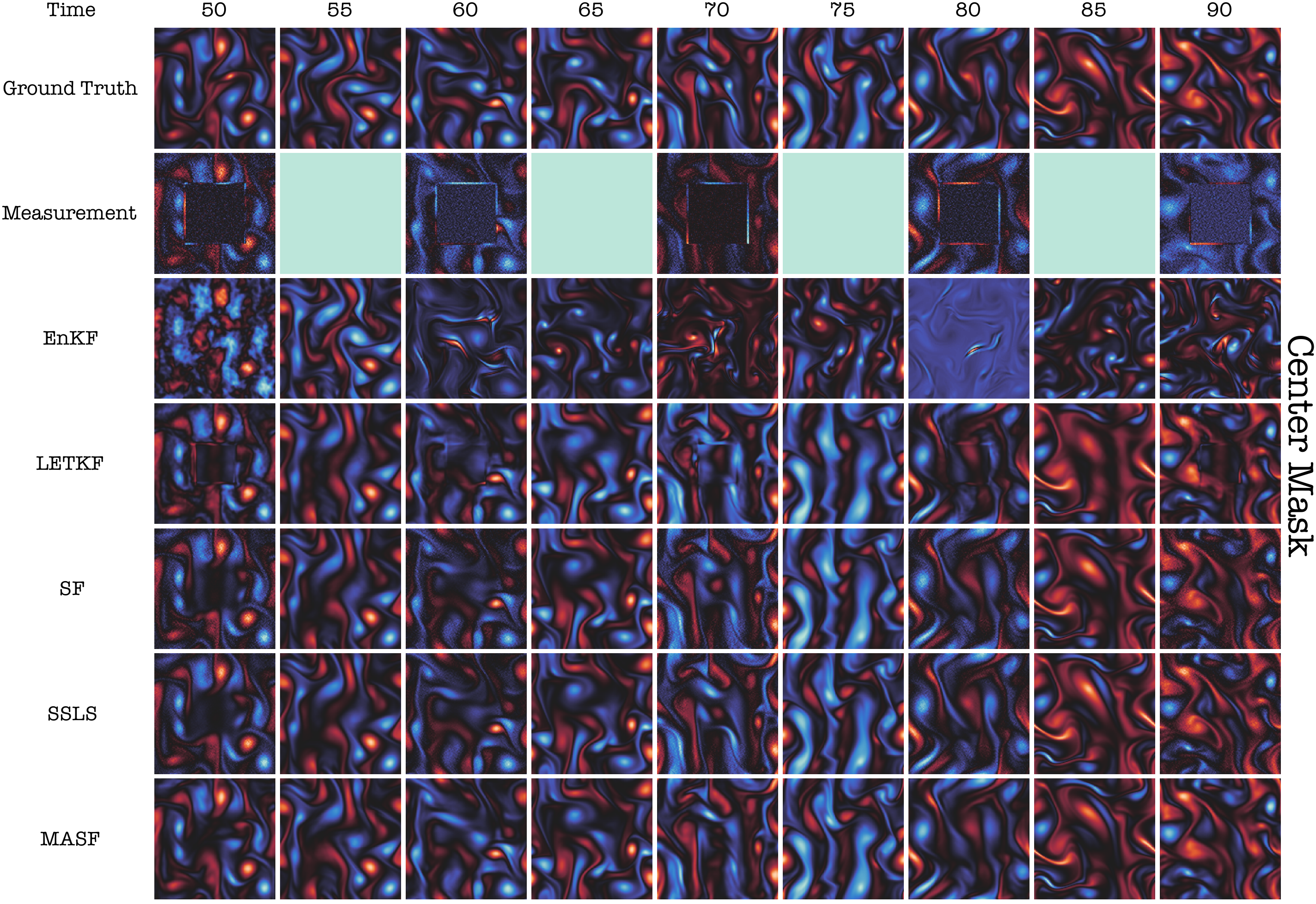}
    \caption{
    \textbf{Center-mask measurements.}
    Qualitative comparison of measurements and posterior samples under center-mask measurements at selected time steps.
    }
    \label{fig:center_mask}
\end{figure}

\begin{figure}[htb!]
    \centering
    \includegraphics[width=1\linewidth]{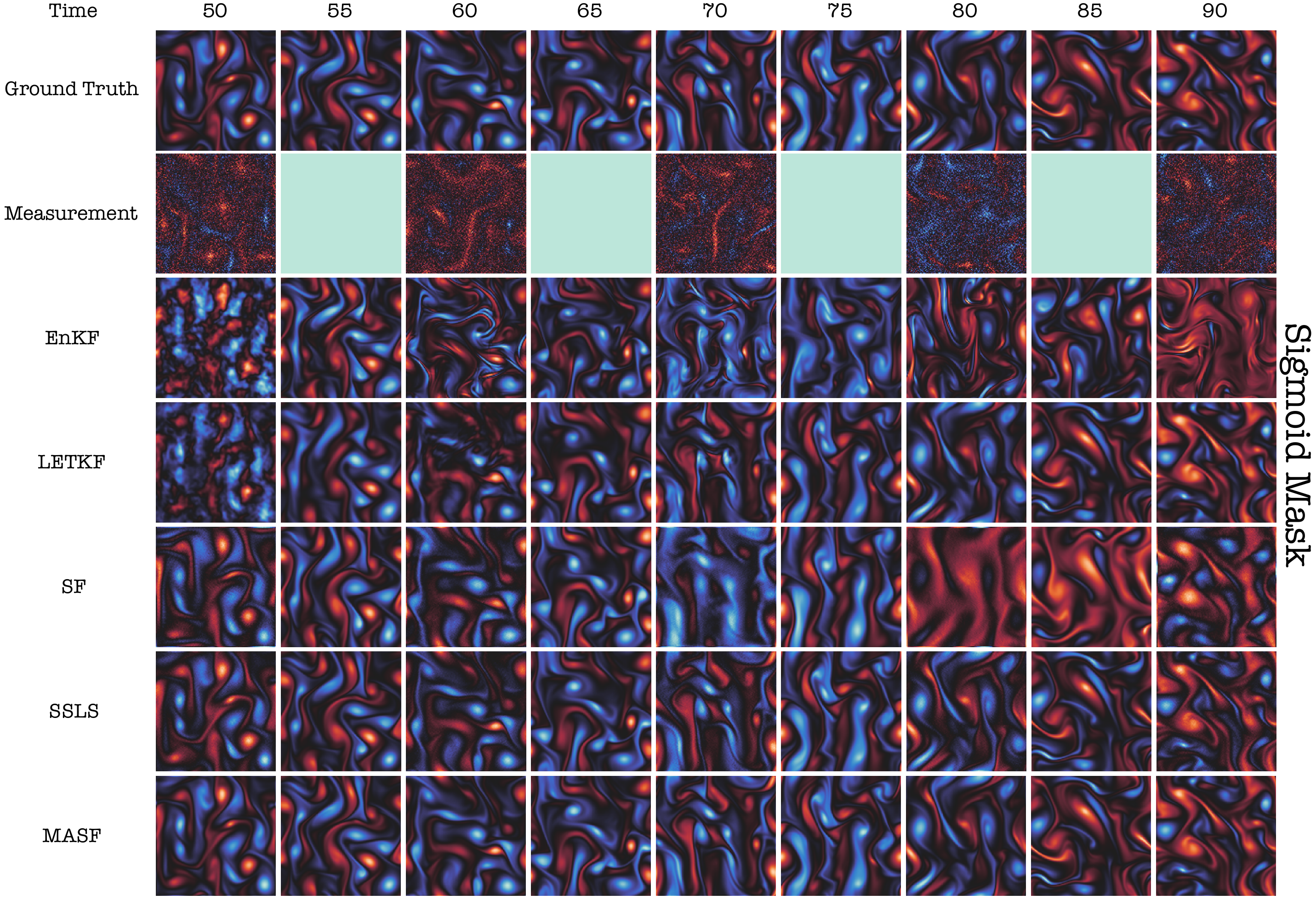}
    \caption{
    \textbf{Sigmoid measurements.}
    Qualitative comparison of measurements and posterior samples under sigmoid measurements at selected time steps.
    }
    \label{fig:sigmoid}
\end{figure}

\begin{figure}[t]
    \centering
    \includegraphics[width=1\linewidth]{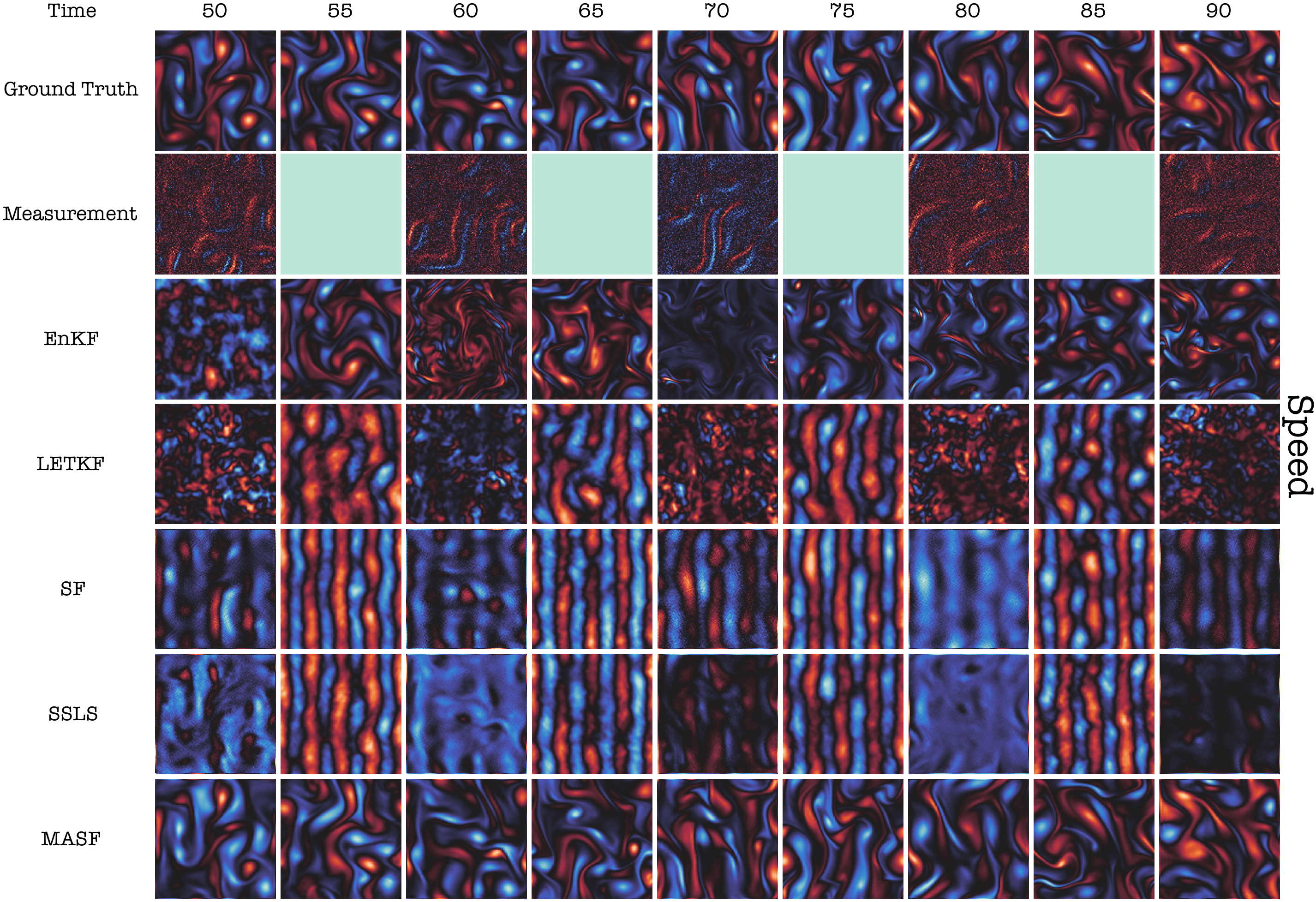}
    \caption{
    \textbf{Speed measurements.}
    Qualitative comparison of measurements and posterior samples under speed measurements at selected time steps.
    }
    \label{fig:speed}
\end{figure}

\clearpage
\newpage

\subsection{Ensemble-Size Sensitivity}
\label{app:ensemble_size_sensitivity}
We analyze how ensemble size affects accuracy and runtime.
Figure~\ref{fig:ensemble_rmse} shows that MASF already reaches a near-saturated accuracy regime with $N=10$ particles, and increasing the ensemble size yields only marginal RMSE improvement while increasing wall-clock time.
In contrast, SF does not improve meaningfully with larger ensembles, and SSLS improves only moderately at substantially higher wall-clock time.
LETKF gradually improves as the ensemble size increases, but its wall-clock time remains high even for small ensembles.
For EnKF, the RMSE improves sharply around $N=250$, so we use $N=250$ as the default EnKF setting in the main experiments.
For LETKF, we use $N=40$ as the default setting because the RMSE improvement beyond $N=40$ is marginal, whereas the wall-clock time continues to increase substantially.
These results show that MASF achieves a favorable runtime--accuracy trade-off by maintaining strong accuracy with a small ensemble.

\begin{figure}[htb!]
    \centering
    \includegraphics[width=1\linewidth]{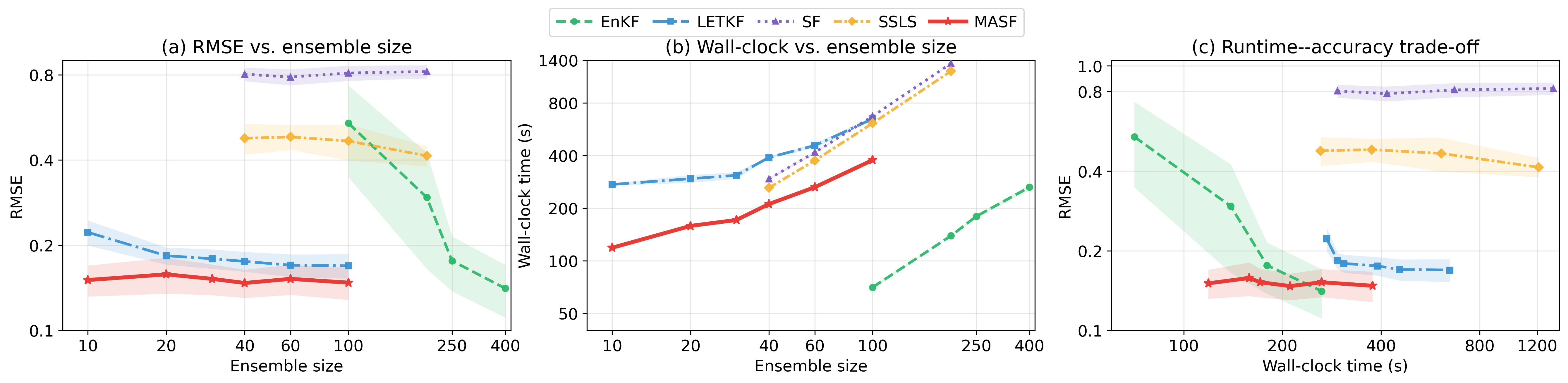}
    \caption{
    \textbf{Ensemble-size sensitivity.}
    Comparison of RMSE, wall-clock time, and runtime--accuracy trade-off as functions of ensemble size for the considered methods.}
    \label{fig:ensemble_rmse}
\end{figure}

\subsection{Extended Temporal Sensitivity Analysis}
\label{app:extended_sensitivity}
We provide additional sensitivity results under grid-mask measurements.
Figure~\ref{fig:app_sensitivity} evaluates sensitivity to temporal length and temporal gap.
MASF maintains low RMSE across these settings.
As the assimilation window becomes longer, MASF gradually improves or remains stable, suggesting that repeated measurement updates refine the state estimate.
When measurements become sparser, all methods degrade, but MASF is less sensitive to the temporal gap than the baselines.
Overall, MASF remains robust under longer horizons and sparse temporal measurements.

\begin{figure}[htb!]
\centering
\includegraphics[width=1\linewidth]{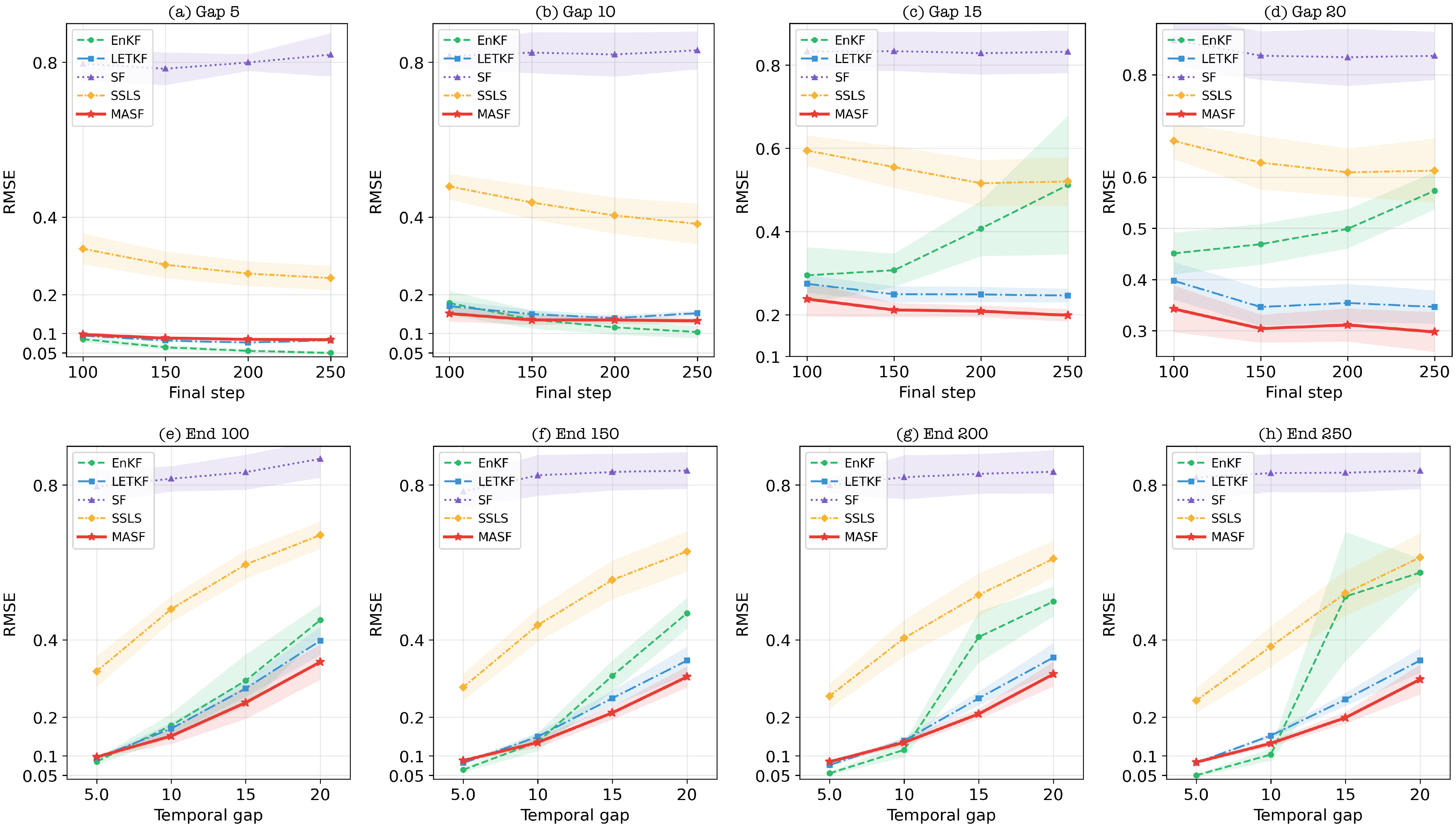}
        \caption{
    \textbf{Extended sensitivity analysis.}
    Sensitivity to temporal length and temporal gap under grid-mask measurements.
    }
    \label{fig:app_sensitivity}
\end{figure}


\end{document}